\documentclass[twoside,journal]{IEEEtran}

\IEEEoverridecommandlockouts                              
\newtheorem{theorem}{Theorem}[section]

\newtheorem{assumption}[theorem]{Assumption}
\newtheorem{problem}{Problem}
\newtheorem{definition}[theorem]{Definition}
\newtheorem{rem}[theorem]{Remark}

\usepackage{graphicx} 
\usepackage{amsmath}
\usepackage{amssymb}
\usepackage{epstopdf}
\usepackage{cite}
\usepackage[noend,ruled,linesnumbered]{algorithm2e}
\SetKwComment{Comment}{$\triangleright$\ } {}
\usepackage{multirow}
\usepackage{rotating}
\usepackage{subfigure} 
\usepackage{color} 
\usepackage{mysymbol}
\usepackage[dvipsnames]{xcolor}
\usepackage[hyphens]{url}

\usepackage[breaklinks=true, colorlinks, bookmarks=true, citecolor=Black, urlcolor=Violet,linkcolor=Black]{hyperref}

\usepackage{comment}
\usepackage[colorinlistoftodos,prependcaption,textwidth=1.5cm,textsize=tiny]{todonotes}
\begin{document}
\title{\LARGE \bf 
Making Infeasible Tasks Feasible: Planning to Reconfigure Disconnected 3D Environments with Movable Objects
}
\author{Samarth Kalluraya, Yiannis Kantaros
\thanks{
S. Kalluraya and Y. Kantaros are with the Department of Electrical and Systems Engineering, Washington University in St. Louis, St. Louis, MO, 63130, USA. This work was supported by the NSF CAREER award CNS \#2340417.
        {\tt\small \{k.samarth,ioannisk@wustl.edu\}}}%
}

\maketitle 
\begin{abstract}
\textcolor{black}{Several planners have been developed to compute \textcolor{black}{dynamically feasible, collision-free} robot paths 
from an initial to a goal configuration. A key assumption in these works is that the goal region is reachable; an assumption that often fails in practice when environments are disconnected. 
Motivated by this limitation, we consider known 3D environments comprising objects, \textcolor{black}{also called blocks, that form distinct navigable support surfaces (planes), and that are either non-movable (e.g., tables) or movable (e.g., boxes). These surfaces may be mutually disconnected} due to height differences, holes, or lateral separations. Our focus is on tasks where the robot must reach a goal region residing on an elevated plane that is  unreachable. Rather than declaring such tasks infeasible, an effective strategy is to enable the robot to interact with the environment, rearranging movable objects to create new traversable connections; a problem known as Navigation Among Movable Objects (NAMO). Existing NAMO planners typically address 2D environments, where obstacles are pushed aside to clear a path. These methods cannot directly handle the considered 3D setting; in such cases, obstacles must be placed strategically to bridge these physical disconnections. We address this challenge by developing \textbf{BRiDGE} (Block-based Reconfiguration in Disconnected 3D Geometric Environments), a sampling-based planner that incrementally builds trees over robot and object configurations to compute feasible plans specifying which objects to move, where to place them, and in what order, while accounting for a limited number of movable objects. To accelerate planning, we introduce non-uniform sampling strategies. We show that our method is probabilistically complete and we provide extensive numerical and hardware experiments validating its effectiveness.}

\end{abstract}

\IEEEpeerreviewmaketitle

\section{Introduction} \label{sec:Intro}
\IEEEPARstart{M}{otion} planning problems focus on computing collision‐free, dynamically feasible robot paths from an initial to a goal configuration~\cite{lavalle2006planning,kaelbling2013integrated,matni2024quantitative}.
\textcolor{black}{While many classical approaches—such as potential/vector fields~\cite{koditschek1990robot,vasilopoulos2022reactive,Vasilopoulos2022Hierarchical}, search‐based planners~\cite{koenig2006new,cohen2010search,nawaz2025graph}, and sampling‐based algorithms~\cite{kavraki1996probabilistic,karaman2011sampling,agha2014firm,janson2015fast,wang2025motion}}—assume that the goal is reachable under a fixed obstacle layout, real-world environments often violate this assumption. In such cases, instead of reporting task failure, robots must actively interact with their environment, rearranging movable objects to \emph{bridge} disconnected areas. 
Achieving this requires planning algorithms that can determine which obstacles to move, in what order, and where to place them. This problem is known as \emph{Navigation Among Movable Objects} (NAMO) \cite{stilman2005navigation} and has been shown to be NP-hard~\cite{Wilfong1988Motion}. 

Several planners have been proposed to address NAMO problems~\cite{stilman2005navigation, Wilfong1988Motion,stilman2007manipulation, stilman2008planning, moghaddam2016planning, nieuwenhuisen2008effective, van2010path, bayraktar2023solving, zhang2025namo,wu2010navigation, levihn2013hierarchical, levihn2013planning, renault2019towards, ellis2023navigation, armleder2024tactile}. These works primarily focus on 2D/planar environments, where movable objects only need to be pushed aside to clear a path to the goal. In other words, once obstacles are removed, the goal region becomes accessible.  Consequently, these methods cannot directly handle NAMO problems in 3D environments with multiple elevated planes; such environments are inherently disconnected due to both the geometry of the environment, i.e., the physical gaps between planes (e.g., step heights, holes on ground, or lateral separations) and the robot dynamics. 
In such settings, reaching the goal may require not only moving obstacles aside but also strategically placing them to act as \emph{bridges} between planes. Thus, simply removing movable objects does not guarantee a pathway. In general, multiple planes may need to be bridged to reach a goal region, and with a limited number of movable objects, some objects may need to be reused to construct these bridges.

\begin{figure}[t]
    \centering
    \includegraphics[width=\linewidth]{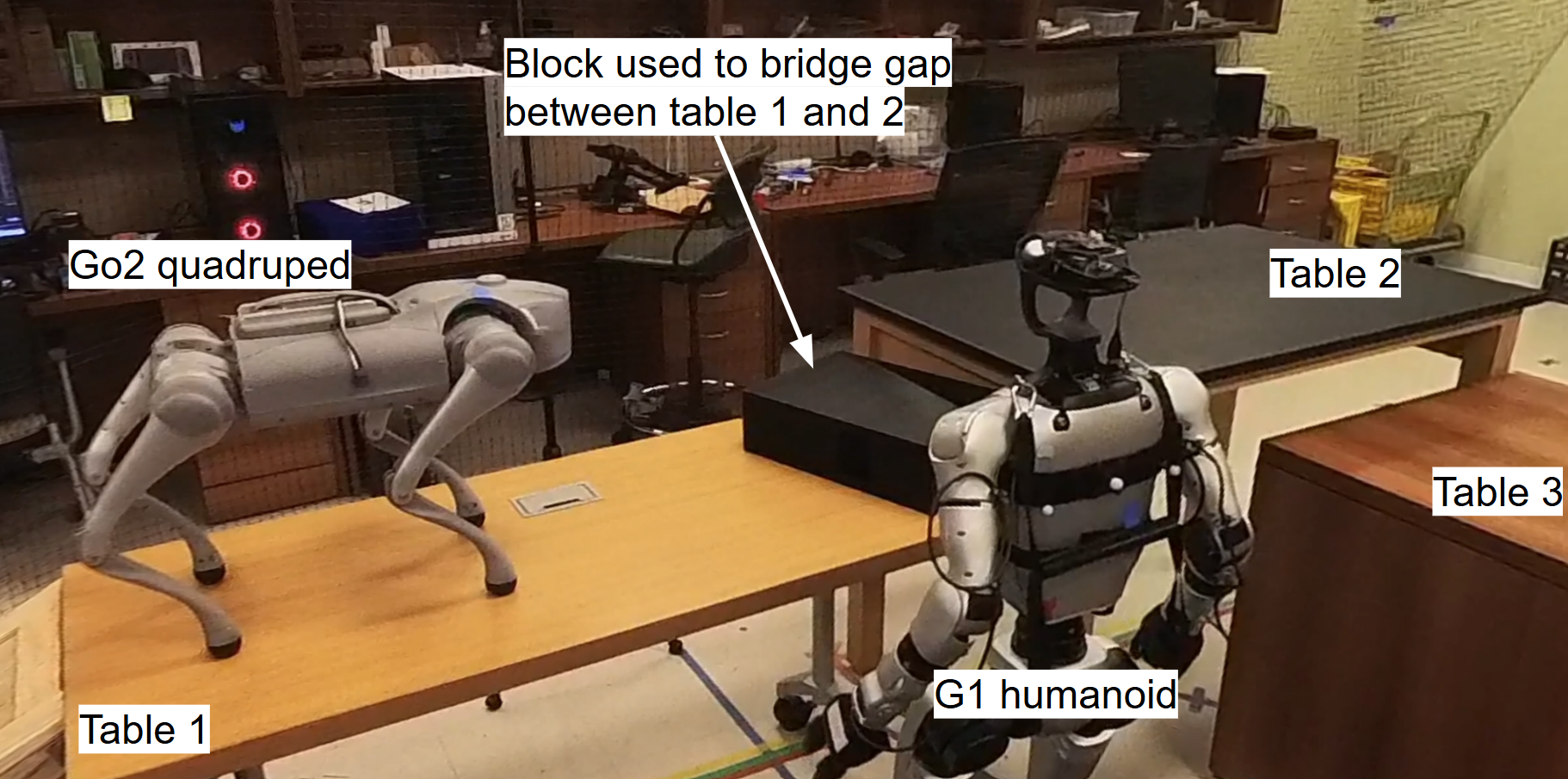}\vspace{-0.2cm}
    \caption{\textcolor{black}{A quadruped robot is tasked with reaching a goal region on Table 3; see Section \ref{subsec:cs3}. However, this requires traversing an elevated plane defined by Table 2. The quadruped’s walking policy does not allow it to climb from Table 1 to Table 2 due to the height difference. Instead of failing the task, a humanoid robot strategically places a block on Table 1 to bridge the gap between the tables for the quadruped.}  }\vspace{-0.4cm}
    \label{fig:motivatingExample}
\end{figure}

In this paper, we study a `\emph{bridge-building}' variant of NAMO in 3D environments. Specifically, we consider environments populated with objects (referred to as blocks) categorized as movable or non-movable (e.g., floor, tables, beds, stools). \textcolor{black}{The environment is assumed to be known, meaning that the locations of all blocks and their movability properties are given.} The top surfaces of these blocks define elevated planes differing in height, orientation, and layout. The robot’s task is to reach a goal region located on one of these surfaces. Due to dynamical constraints, such regions may not be directly reachable; for example, a quadruped robot may be unable to jump from the floor to a desk to access a desired object. In these cases, the robot must use movable objects as assets (and not just as obstacles to be moved out of the way) by rearranging them to bridge planes and thereby make the goal region accessible; \textcolor{black}{see also Fig. \ref{fig:motivatingExample}.}
\textcolor{black}{
A particular challenge in this problem setting is that objects required to bridge such gaps may themselves be inaccessible. For example, the robot may need a stool to climb onto a desk, but that stool may lie on a different elevated plane. In this case, the robot must first reason about how to use other reachable objects to access the stool. The difficulty increases when only a small number of reachable objects are available, requiring that the same objects be used strategically and possibly multiple times to bridge different gaps. These dependencies across obstructions/gaps necessitate long-horizon pick-and-place plans.
}

To address this challenge, we \textcolor{black}{introduce \textbf{BRiDGE} (Block-based Reconfiguration in Disconnected 3D Geometric Environments),} a sampling-based planning framework that incrementally builds trees over robot and block configurations to compute feasible long-horizon manipulation plans specifying which blocks to move, where to place them to bridge planes, and in what order. \textcolor{black}{To accelerate planning, \textcolor{black}{\textbf{BRiDGE} incorporates} non-uniform sampling strategies that prioritize tree growth toward potentially promising directions. These directions are determined by a high-level planner that generates symbolic plans specifying a sequence of pick-and-place actions for blocks, without reasoning about the geometric or dynamical feasibility of these actions or determining the exact placement poses of the blocks. In our implementation, we employ a Breadth-First-Search (BFS) planner and Large Language Models (LLMs) to generate these high-level plans, though other planners such as A* could also be used.
We emphasize that the symbolic plan may be infeasible or non-executable, e.g., because the high-level planner is not sound (e.g., LLMs) or because it cannot account for geometric and dynamical constraints (e.g., BFS or LLMs). Such incorrect plans may guide the tree into infeasible directions, potentially slowing down planning. To mitigate this, we provide mechanisms to refine the symbolic plan by re-invoking the high-level planner to replan and removing high-level actions that proved infeasible.
We show that \textcolor{black}{\textbf{BRiDGE}} is probabilistically complete meaning that the probability of finding a solution, if one exists, approaches one as the number of iterations of the sampling-based planner increases. This guarantee holds regardless of the choice of symbolic planner used to guide the sampling strategy. We provide extensive experiments demonstrating that our algorithm can solve complex, long-horizon planning problems with a limited number of objects, often requiring re-use of blocks to bridge multiple gaps. Moreover, we show empirically that our non-uniform sampling strategies significantly accelerate planning compared to uniform sampling. Finally, we conduct hardware experiments showing how \textcolor{black}{\textbf{BRiDGE}}, originally designed for a single robot, can be applied to a team of two robots (a quadruped and a humanoid) with complementary navigation and manipulation skills that collaborate to  render the goal region accessible to one of them; see Fig. \ref{fig:motivatingExample}.}

\begin{figure*}[t]
    \centering
    \subfigure[]{\label{fig:case1a}\includegraphics[width=0.23\textwidth]{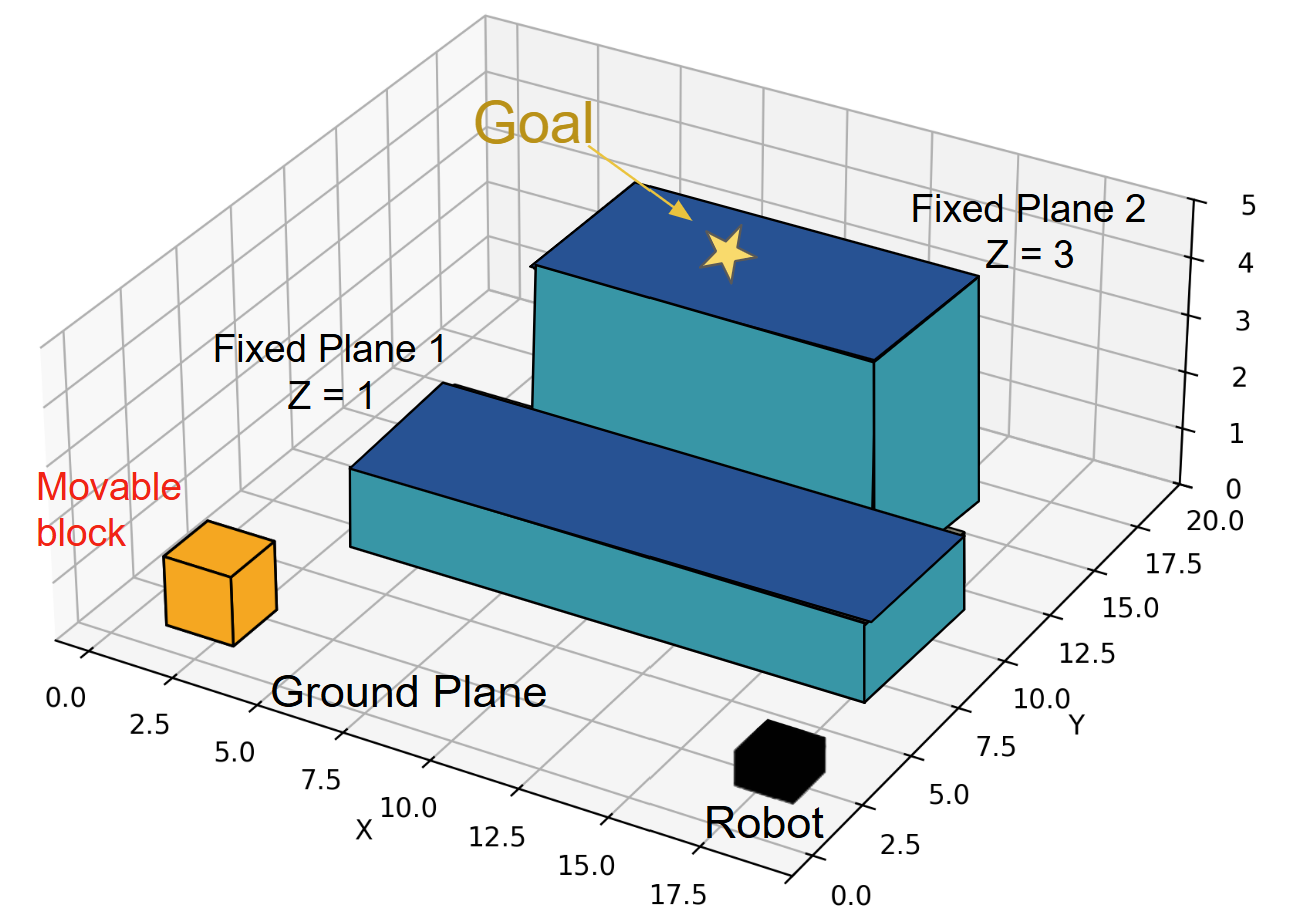}} \hfill
    \subfigure[]{\label{fig:case1b}\includegraphics[width=0.23\textwidth]{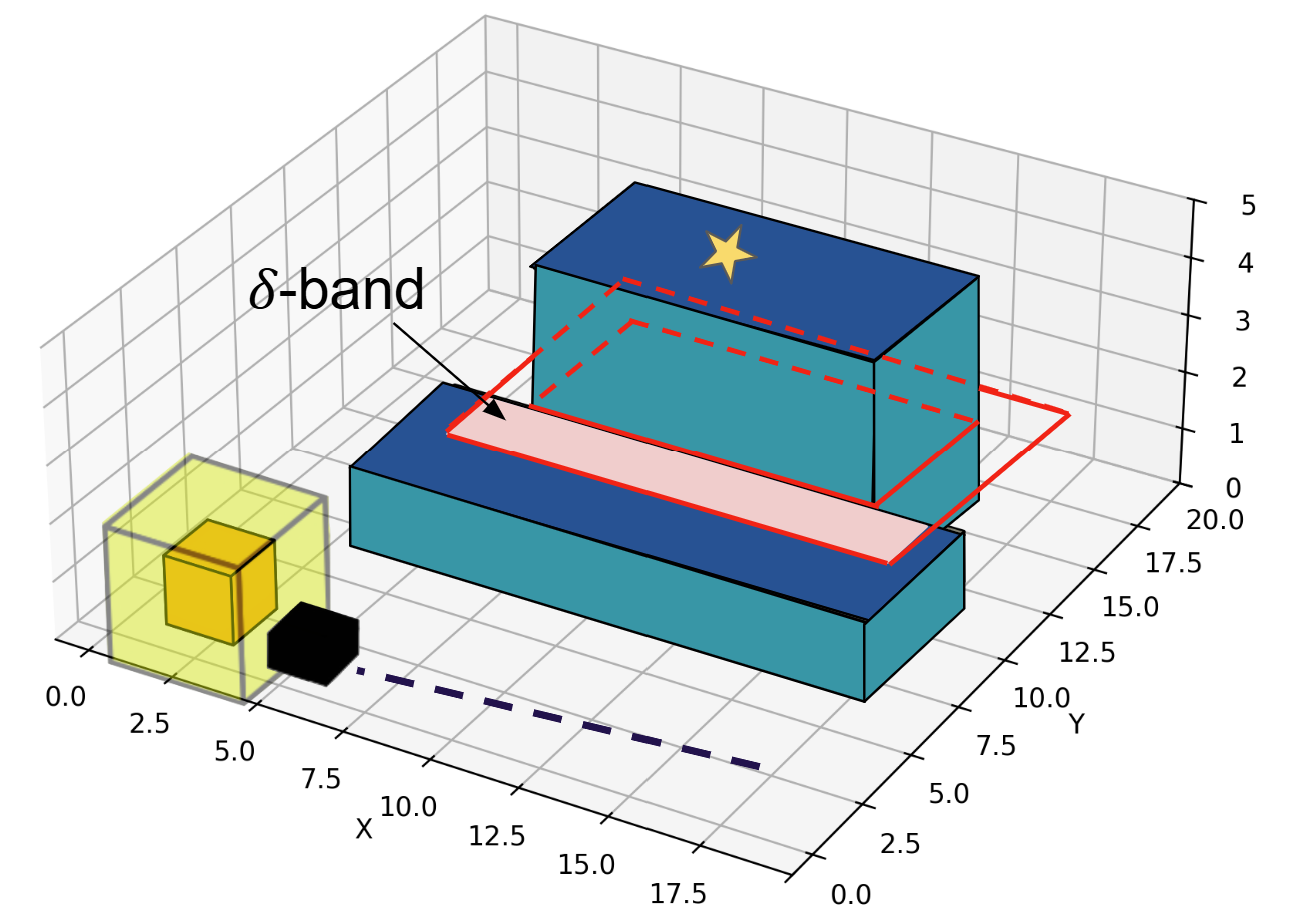}} \hfill
    \subfigure[]{\label{fig:case1c}\includegraphics[width=0.23\textwidth]{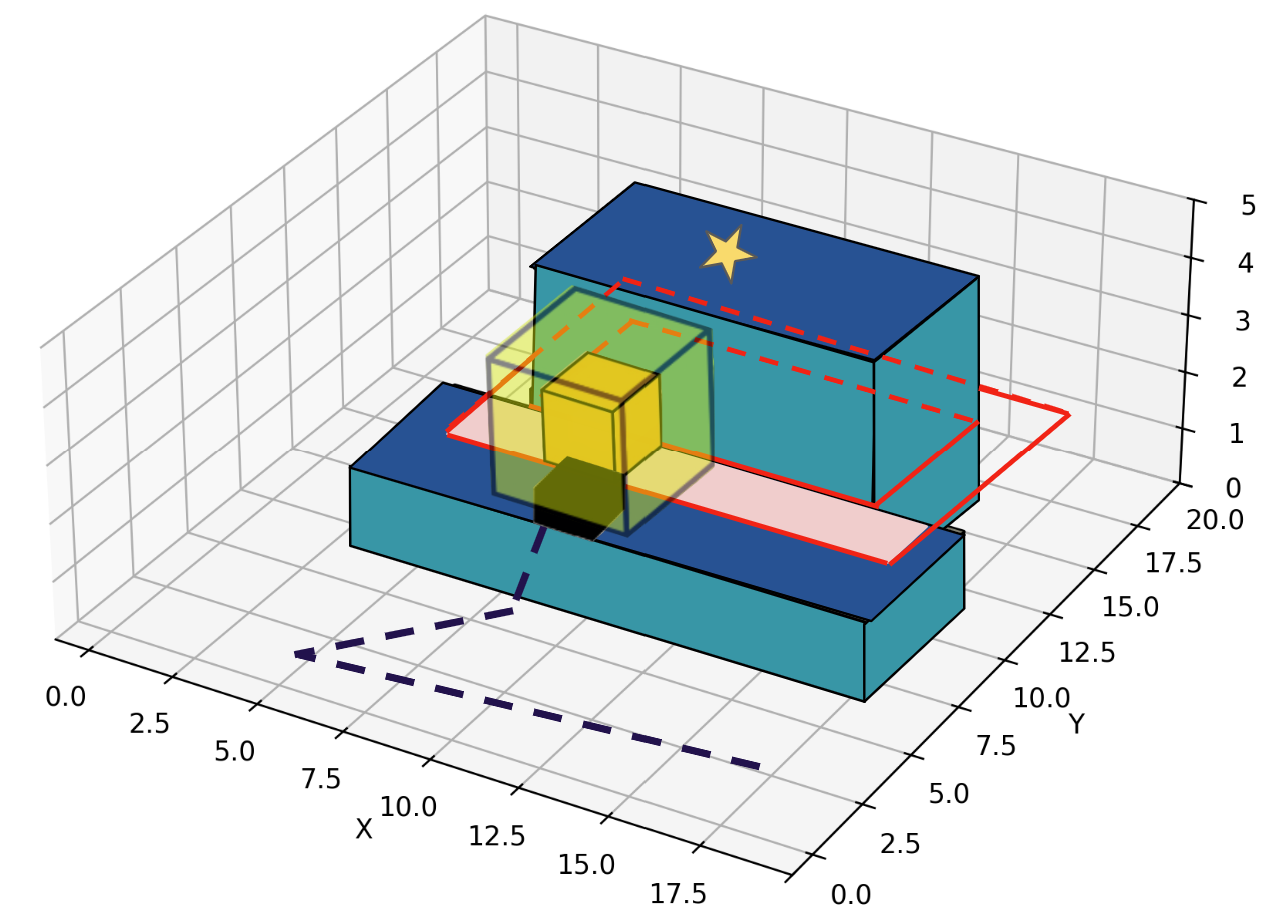}} \hfill
    \subfigure[]{\label{fig:case1d}\includegraphics[width=0.23\textwidth]{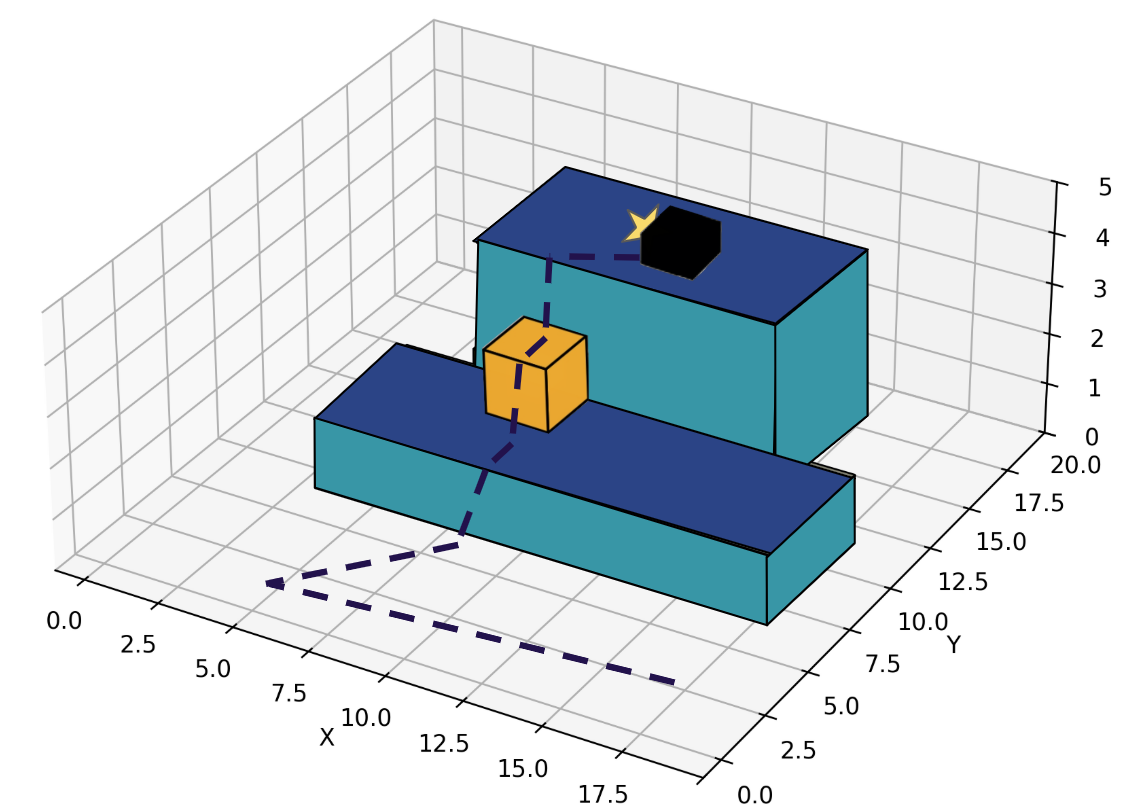}} 
    \caption{Example environment with three fixed (non-movable) planes—Ground ($z{=}0$), Plane 1 ($z{=}1$), and Plane 2 ($z{=}3$)—and a single movable block $b_1$ with height $h_1 = 1$ unit. The robot (black cube) is initially located on the ground plane and is tasked with reaching a goal state (yellow star) on Plane 2. As per \texttt{NavigReach}, the robot can traverse vertical height differences of up to $1$ unit. \textcolor{black}{Observe that the robot cannot reach the goal region from its initial location shown in Fig.~\ref{fig:case1a}. While Plane 1 is reachable from the ground, Plane 2 is not reachable from Plane 1 because their height difference is $2$ units, and it is also not directly reachable from the ground since their height difference is $3$ units.}
    Thus, the robot first approaches and grasps the block $b_1$, as illustrated in Fig.~\ref{fig:case1b}. The translucent green box visualizes the grasp-feasible region generated by \texttt{ManipReach}. The $\delta$-band (pink region) highlights a region on Plane 1 such that, if $b_1$ is placed there, Plane 2 may become reachable from Plane 1 (see Sec.~\ref{sec:block_placement_density}). After grasping $b_1$, the robot carries it onto Plane 1 and places it within the pink region (Fig.~\ref{fig:case1c}); the green box now visualizes the feasible drop-off region determined by \texttt{ManipReach}. Using the placed block as a step, the robot climbs onto the block and then onto Plane 2, reaching the goal (Fig.~\ref{fig:case1d}). Dashed black lines illustrate the robot’s path.}

    \label{fig:env}
\end{figure*}

\textbf{Related Works:} Next, we review related works in four areas: (i) NAMO; (ii) rearrangement planning; (iii) multi-robot mutualism; and (iv) non-uniform sampling strategies.

\textit{(i) NAMO:} Several search- and sampling-based planners have been developed to address NAMO problems~\cite{stilman2005navigation, stilman2007manipulation, stilman2008planning, moghaddam2016planning, nieuwenhuisen2008effective, van2010path, bayraktar2023solving, zhang2025namo}. Extensions to unknown environments have been investigated in~\cite{wu2010navigation, levihn2013hierarchical, levihn2013planning, renault2019towards, ellis2023navigation, armleder2024tactile}. These works focus on 2D planar navigation and, therefore, movable objects are modeled primarily as obstacles that should be pushed out of the way to reopen a blocked path on a single support surface (e.g., the floor). By contrast, our setting treats objects as resources to create traversable connections among multiple support surfaces; we leverage blocks to form steps/bridges that alter inter-plane connectivity in 3D environments.
\textcolor{black}{The NAMO problem most closely related to ours is addressed in \cite{Levihn2014Autonomous} and its extension \cite{Levihn2015Using}. These papers address navigation obstructions via reactive tool use in surface-constrained environments: when encountering an unavoidable obstacle (e.g., a hole in the floor), the robot identifies an appropriate object through inverse affordances and deploys it to restore connectivity. A key assumption in these works is that obstructions can be resolved \textit{independently}, each with a \textit{single always-available object}. Consequently, the repair mechanism can be invoked repeatedly in a purely local manner, without coordinating actions across multiple obstructions. This assumption does not hold in our setting, where obstructions (i.e., physical gaps between elevated planes) can be \textit{dependent} and must be handled \textit{coordinately} due to the limited number of available blocks. In other words, in \cite{Levihn2015Using}, each obstruction can be resolved with a single pick-and-place action, whereas our setting requires long-horizon pick-and-place plans. For example, reaching a goal region may require bridging gap X, but gap X can only be bridged using block A that is located on an inaccessible plane Y. Solving such a case requires the robot to reason about how to use its reachable blocks (possibly multiple times) to first access plane Y, retrieve block A, and then deploy it to bridge gap X. Such planning capabilities are not developed in \cite{Levihn2015Using}; instead, in the previous example, that work would assume that block A bridging the gap X is accessible to the robot. Additionally, unlike \cite{Levihn2015Using}, our planner is supported by probabilistic completeness guarantees.} 

\textit{(ii) Rearrangement Planning (RP):} Related work also exists in RP, where the robot interacts with movable objects to achieve specific object‐goal configurations~\cite{labbe2020monte,krontiris2015dealing,ahn2024relopush,ren2022rearrangement}. Unlike NAMO, RP problems typically do not specify a goal for the robot itself, but instead focus on realizing a predefined final arrangement of the objects. In contrast, in our setting, the robot must reach a designated goal region, while the final configuration of the obstacles is not pre-specified. The robot must reason about how to rearrange obstacles to enable its own traversal to the goal, rather than simply arranging objects according to a fixed target configuration. 

\textcolor{black}{\textit{(iii) Mutualism in Multi-Robot Systems:} Another line of related work explores how to leverage heterogeneity in robot teams to complete tasks that individual robots could not accomplish alone \cite{nguyen2023mutualistic,nguyen2025mutualisms,nguyen2024resiliency,nguyen2024scalable,egerstedt2021robot}. For example, consider a ground robot that must reach a goal area separated from its initial position by a water body. Instead of declaring the task infeasible, these works develop coordination frameworks in which robots help each other—e.g., an aquatic robot carries the ground robot—to complete the task. However, in these approaches, robots remain passive with respect to the environment’s structure. In the previous example, if no aquatic robot is available, the task is considered infeasible. In contrast, our framework allows robots to actively interact with and rearrange the environment itself to connect previously disconnected areas; e.g., by building a `bridge' to cross the water body.}

\textit{(iv) Non-uniform Sampling Strategies:}
Several works design \emph{non-uniform sampling strategies} to accelerate sampling-based planners~\cite{luo2021abstraction,wang2020neural,ichter2018learning,liu2024nngtl,johnson2023learning,kantaros2020stylus,qureshi2020motion,feng2025rrt,natraj2025conformalized}. However, these approaches target different planning problems, focusing on finding obstacle-free paths between an initial and goal robot configuration assuming such paths exist and thus they cannot be directly applied to our setting. Non-uniform sampling for NAMO problems is considered in~\cite{zhang2025namo}, which, however, focuses on 2D environments and therefore does not trivially extend to our problem.

\textbf{Contributions:}
\textit{First}, we introduce a new NAMO problem in disconnected 3D environments in which movable objects are treated as \textit{assets} for bridging disconnected planes, rather than merely as obstacles to be removed.
\textit{Second}, we propose a novel sampling‐based algorithm that jointly explores the configuration space of the robot and objects to determine which objects to move, where to place them, and in what order, thereby enabling goal reachability.
\textit{Third}, we design non‐uniform sampling strategies that prioritize tree growth toward promising directions.
\textit{Fourth}, we prove that our algorithm is probabilistically complete.
\textit{Finally}, we provide extensive numerical and hardware results that validate the effectiveness and efficiency of our approach.

\section{Problem Definition} \label{sec:PF}

\subsection{Environment}\label{sec:PFEnvironment} 
We consider an environment in $\mathbb{R}^3$ containing a set $\ccalB=\{b_1,\dots,b_{N_b}\}$ of $N_b \geq 1$ known objects, called \textit{blocks}. For simplicity, we model these blocks as right polygonal prisms (see Remark~\ref{rem:shapes}). The geometric configuration of each block $b_i$ is specified by the coordinates $\bbq_i^j$ of the vertices $j\in\{1,\dots,s\}$ of its bottom surface, together with its height $h_i>0$. We compactly denote the geometric configuration of $b_i$ as a tuple:
\begin{equation}\label{eq:block_config}
\bbb_i = (\bbQ_i, h_i),
\end{equation}
where $\bbQ_i = [\bbq_i^1, \bbq_i^2, \dots, \bbq_i^s] \in \mathbb{R}^{s \times 3}$ is the matrix collecting the coordinates of the $s$ vertices of the bottom surface. 


The blocks are divided into (i) \textit{non-movable blocks}, collected in $\ccalB^{\text{nm}}\subseteq\ccalB$, representing fixed structures in the environment (e.g., floors, tables, or beds) whose configuration $\bbb_i$ does not change over time; and (ii) \textit{movable blocks}, $\ccalB^{\text{m}}\subseteq\ccalB$, i.e., objects whose position and orientation (captured by the vertices of their bottom surface in $\bbb_i$) can be modified by the robot. The top surface of each block is assumed to be navigable by the robot. We call these surfaces \textit{planes}. Formally, the plane associated with block $b_i$ is defined as
\begin{equation}\label{eq:plane}
\ccalP_i = \{\bbq \in \mathbb{R}^3 \mid g_j(\bbq)\leq 0, \forall j\in\{1,\dots,s\}\},
\end{equation}
where $g_j(\bbq)$ are linear inequalities for the inward half-spaces bounded by the edges of block’s top-surface polygon.\footnote{Each $g_j(\bbq)$ takes the form $g_j(\bbq) = \mathbf{a}_j \cdot (\bbq-\bbq_0) + \mathbf{b}_j$, where $\mathbf{a}_j$ is the in-plane normal to edge $j$, $\bbq$ is the test point, and $\bbq_0$ is a reference point on the plane.}
\textcolor{black}{Throughout the paper, for simplicity of presentation, we also assume that the top surface of a movable block $b_i$ (and consequently its height $h_i$) remains fixed. This implies that while the location and orientation (i.e., $\bbQ_i$) of a movable block $b_i$ may change, the block cannot be flipped.\footnote{\textcolor{black}{This assumption can be relaxed by allowing the proposed sampling-based planner to search not only over the continuous space of block poses (captured by $\bbQ_i$) but also over a discrete space of block heights.}}}

Planes associated with non-movable blocks are collected in the set $\ccalP^{\text{nm}} = \{\ccalP_i \mid b_i\in\ccalB^{\text{nm}}\}$. In the set $\ccalP^{\text{nm}}$ we also include the ground plane, \textcolor{black}{denoted by $\ccalP_G$,} that is not induced by any block $b_i\in\ccalB$.
%
Similarly, we define $\ccalP^{\text{m}} = \{\ccalP_i \mid b_i\in\ccalB^{\text{m}}\}$ as the set of planes associated with movable blocks. 
The structure of the environment at time $t$, determined by all blocks $b_i$, is denoted by $\Omega(t)$. An example of such an environment is shown in Fig.~\ref{fig:env}.




\subsection{Robot Dynamics}\label{sec:robDynamics}
We consider a holonomic robot whose configuration is fully determined by its position $\bbx \in \ccalX \subseteq \mathbb{R}^3$, where $\ccalX$ denotes the robot’s configuration (position) space and consists of the traversable region on all planes ($\ccalP^{\text{nm}}$ and $\ccalP^{\text{m}}$). Its low-level motion dynamics are abstracted by a function that checks whether a dynamically feasible, collision-free path exists between two positions, defined as follows.

\begin{definition}[Function $\texttt{NavigReach}$]\label{def:check}
The function $\texttt{NavigReach}:\ccalX \times \ccalX \times \Omega\rightarrow \{\texttt{True}, \texttt{False}\}$ takes as input two robot positions $\bbx_i,\bbx_j \in \ccalX$ and the current environment state $\Omega(t)$. It returns \texttt{True} if there exists a dynamically feasible, collision-free path from $\bbx_i$ to $\bbx_j$ in $\Omega(t)$; otherwise, it returns \texttt{False}.
\end{definition}

The robot is also equipped with a manipulator that allows it to pick and place blocks. 
We abstract the low-level arm dynamics into the following high-level function:
\begin{definition}[Function $\texttt{ManipReach}$]\label{def:manipreach}
The function $\texttt{ManipReach} : \ccalB^{\text{m}} \times \Omega \rightarrow \ccalX$ takes as input a movable block $b_i \in \ccalB^{\text{m}}$ and an environment state $\Omega(t)$, and returns the set $\ccalX_{\text{manip}} \subseteq \ccalX$ of robot configurations from which the robot can feasibly interact with $b_i$, subject to its kinematic and dynamic constraints in $\Omega(t)$. 
\textcolor{black}{Therefore, given the configuration $\bbb_i$ of block $b_i$ in $\Omega(t)$, the function returns all robot configurations from which the robot can either grasp $b_i$ or place $b_i$ in that same configuration $\bbb_i$.}
\end{definition}
Implementation details for $\texttt{NavigReach}$ and $\texttt{ManipReach}$ are provided in Sec. \ref{sec:PRM}.

\subsection{Task Specification}\label{sec:Task}
The robot is initially located on one of the planes $\ccalP_i$ and is tasked with reaching a \textcolor{black}{designated goal state $\bbx_{\text{goal}}\in \ccalX$} that lies on a plane $\ccalP_j$ ($\ccalP_i \neq \ccalP_j$). We focus on cases where $\bbx_{\text{goal}}$ is not reachable by the initial robot position $\bbx_{\text{start}}$, due to the robot dynamics not allowing it to directly move or jump to desired planes; i.e., $\texttt{NavigReach}(\bbx_{\text{start}}, \bbx_{\text{goal}}, \Omega(0)) = \texttt{False}$, 
Instead of reporting task failure, the robot must determine how to reconfigure the environment.

To formalize this, we define a manipulation action
$\texttt{Move}(\bbb_i, \bbb_i'),$
which relocates a block $b_i$ from its current configuration $\bbb_i$ to a new configuration $\bbb_i'$. Executing such an action requires the robot to navigate to a configuration in $\ccalX_{\text{manip}}\textcolor{black}{=\texttt{ManipReach}(b_i,\Omega(t))}$ from which $b_i$ can be grasped, transport it, and place it at $\bbb_i'$. To preserve the movability properties of all blocks, we require movable blocks to be placed only on non-movable planes \textcolor{black}{and that the movable block can fully lie on top of the non-movable one}.
The objective is to compute a manipulation plan $\tau=\langle \tau(1),\dots,\tau(k),\dots,\tau(H)\rangle$ where each $\tau(k)=\texttt{Move}(\bbb_i,\bbb_i')$ denotes the $k$-th manipulation action, $k\in\{1,\dots,H\}$ for some horizon  $H\geq1$. After executing $\tau$, the robot reaches a state $\bbx'$ in $\Omega(T)$ from which $\bbx_{\text{goal}}$ is reachable, i.e., $\texttt{NavigReach}(\bbx', \bbx_{\text{goal}}, \Omega(T))=\texttt{True}$ for some $T>0$. \textcolor{black}{Note that $k\in\{1,\dots,H\}$ indexes actions in the plan $\tau$, whereas $t\in[0,T]$ denotes continuous time. For example, when the $H$-th action is executed, $T$ time units have elapsed.}

\begin{problem}\label{prob_statement} 
Given $\texttt{NavigReach}$, $\texttt{ManipReach}$, an initial position $\bbx_{\text{start}}$, an initial environment state $\Omega(0)$, and \textcolor{black}{a goal state $\bbx_{\text{goal}} \in \ccalX$, compute a manipulation plan
$\tau$ such that after execution, the robot can reach  $\bbx_{\text{goal}}$.}
\end{problem}

\begin{rem}[Complex Block Structures]\label{rem:shapes}
\textcolor{black}{Our approach can support blocks with arbitrary geometry (e.g., non-flat top surfaces or non-vertical sides) as long as \texttt{NavigReach} and \texttt{ManipReach} can handle such cases. In particular, \texttt{NavigReach} must be able to reason about the traversability of complex surfaces and terrains, while \texttt{ManipReach} must be able to assess the static stability of objects placed on uneven terrains. A similar requirement applies to the function \texttt{Candidate}, introduced in Section~\ref{sec:settingUP}. Extending our framework to more unstructured environments by designing these high-level abstractions is outside the scope of this work and will be pursued in future research.}
\end{rem}

\section{\textcolor{black}{Interactive Planner \\For Environment Reconfiguration}}\label{sec:ResilientPlanning}

In this section, we present \textcolor{black}{\textbf{BRiDGE}}, our sampling-based planning framework for solving Problem~\ref{prob_statement}. First, in Section~\ref{sec:settingUP} we introduce key definitions and utility functions used by \textcolor{black}{\textbf{BRiDGE}}. Then, in Section~\ref{sec:planner} we describe the core sampling-based planner that generates a manipulation plan $\tau$ solving Problem~\ref{prob_statement}.

\subsection{Setting up the Planner}\label{sec:settingUP}
We first introduce key definitions and utility functions that will be used in \textcolor{black}{\textbf{BRiDGE}}, complementing Definitions \ref{def:check} and \ref{def:manipreach}. Specifically, we define when two planes $\ccalP_i$ and $\ccalP_j$ are pairwise inaccessible, i.e., when the robot cannot move from $\ccalP_i$ to $\ccalP_j$ without violating dynamic constraints or entering obstacle regions.

\begin{definition}[Pairwise Inaccessible Planes]\label{def:disjointplanes}
Let $\ccalP_i$ and $\ccalP_j$ be any two distinct planes. We say that $\ccalP_j$ is inaccessible/unreachable from $\ccalP_i$ if $\texttt{NavigReach}(\bbx_i, \bbx_j, \Omega(t)) = \texttt{False}$ for all pairs of robot positions $\bbx_i, \bbx_j \in \ccalX$ whose corresponding points lie on $\ccalP_i$ and $\ccalP_j$, respectively.
\end{definition}

For simplicity of presentation, we assume that this relation is symmetric, i.e., if $\ccalP_j$ is inaccessible from $\ccalP_i$, then the reverse holds as well. 
%
If $\ccalP_j$ is inaccessible from $\ccalP_i$, then we say that there exists a `gap' between them.
\textcolor{black}{We are concerned with gaps in the structural layout induced by \emph{non-movable} planes only. That is, we consider gaps between $\ccalP_i\in\ccalP^{\text{nm}}$ and $\ccalP_j\in\ccalP^{\text{nm}}$ that will be bridged using movable blocks. Gaps between non-movable blocks are identified using the following function:}

\begin{definition}[Function $\texttt{Gap}$]\label{def:Gap}
The function $\texttt{Gap}:\ccalP^{\text{nm}} \times \ccalP^{\text{nm}} \times \Omega \rightarrow \{\texttt{True}, \texttt{False}\}$ returns \texttt{True} iff there exists a gap between $\ccalP_i\in\ccalP^{\text{nm}}$ and $\ccalP_j\in\ccalP^{\text{nm}}$ (as in Definition~\ref{def:disjointplanes}); otherwise, it returns \texttt{False}.
\end{definition}


Next, we introduce a function that determines which movable blocks can potentially be used, by relocating them, to `bridge' a gap between $\ccalP_i$ and $\ccalP_j$.

\begin{definition}[Function \texttt{Candidate}]\label{def:Candidate}
The function $\texttt{Candidate}:(\ccalP^{\text{nm}}, \ccalP^{\text{nm}}) \rightarrow \{ (b_n, \ccalP_k) \mid b_n \in \ccalB^{\text{m}}, \, \ccalP_k \in \ccalP^{\text{nm}} \}$
takes as input a pair of pairwise inaccessible non-movable planes $(\ccalP_i, \ccalP_j\in \ccalP^{\text{nm}})$ and returns the set of \emph{all} tuples $(b_n, \ccalP_k)$ such that there exists at least one placement configuration of $b_n$ on $\ccalP_k$ that, if realized, \textcolor{black}{may potentially} enable the robot (as determined by $\texttt{NavigReach}$) to reach a configuration on $\ccalP_j$ starting from a configuration on $\ccalP_i$, using the top surface of $b_n$ as an intermediate plane.
\end{definition}

\textcolor{black}{Note that \texttt{Candidate} does not return the specific placement configuration of the movable block $b_n$ on the non-movable plane $\ccalP_k$. The function also does not need to be perfect: it may return tuples $(b_n, \ccalP_k)$ for which no placement of $b_n$ on $\ccalP_k$ actually bridges the desired gap. Such incorrect candidates are later filtered out by feasibility checks (see Alg.~\ref{alg:sample_move}). We make the following assumption for this function:}

\begin{assumption}[Candidate Over-Approximation]\label{ass:cand}
\textcolor{black}{
The set returned by \texttt{Candidate}$(\ccalP_i,\ccalP_j)$ must include every tuple $(b_n,\ccalP_k)$ that can enable a gap-bridging placement between $(\ccalP_i,\ccalP_j)$. In other words, \texttt{Candidate} may include false positives, but it must not exclude any true gap-bridging option.}
\end{assumption}

\textcolor{black}{A conservative design for \texttt{Candidate} to satisfy Assumption \ref{ass:cand} is to return 
all possible tuples $(b_n,\ccalP_k)$. In Sec.~\ref{sec:PRM}, we describe a less 
conservative implementation of the \texttt{Gap} and \texttt{Candidate} functions,
which leverages the right–polygonal-prism structure of all blocks.}

\subsection{Sampling-Based Planner}\label{sec:planner}
%
To compute plans $\tau$ addressing Problem~\ref{prob_statement}, we propose a sampling-based approach that incrementally builds a tree exploring the space of robot and block configurations. The proposed algorithm is summarized in Alg.~\ref{alg:sampling_planner}.

 In what follows we provide intuition for the steps of the \textcolor{black}{\textbf{BRiDGE} planner} Alg.~\ref{alg:sampling_planner}. The tree is defined as $\ccalT=(\ccalV_T,\ccalE_T)$, where $\ccalV_T$ and $\ccalE_T\subseteq\ccalV_T\times\ccalV_T$ are the sets of nodes and directed edges, respectively. A node $v\in\ccalV_T$ is written as $v = \big[\Omega^v,\; \bbx^v\big]$, where $\Omega^v$ denotes the environment structure (i.e., the configurations of all non-movable and movable blocks) and $\bbx^v\in\ccalX$ is the robot position associated with that node. A directed edge $(v,v')\in\ccalE_T$ exists if applying a single manipulation $\texttt{Move}(\bbb_i,\bbb_i')$ to some block $b_i$ transforms $\Omega^v$ into $\Omega^{v'}$. With a slight abuse of notation we write this reconfiguration as $\Omega^{v'} = \Omega^v \oplus \texttt{Move}(\bbb_i,\bbb_i')$
where $\oplus$ denotes the environment update that replaces $b_i$'s configuration $\bbb_i$ by $\bbb_i'$. The tree is rooted at $v_{\text{root}}=[\Omega(0),\bbx_{\text{start}}]$, with initial sets $\ccalV_T=\{v_{\text{root}}\}$ and $\ccalE_T=\emptyset$ (line~\ref{alg:init_tree}, Alg.~\ref{alg:sampling_planner}). The tree grows iteratively by alternating sampling and extension steps. At each iteration $\kappa$ we (i) sample an existing tree node $v_{\text{rand}}\in\ccalV_T$, (ii) sample a manipulation action $\texttt{Move}(\bbb_i,\bbb_i')$ for some movable block $b_i$, and (iii) attempt to apply that action to $v_{\text{rand}}$ to obtain a candidate node $v_{\text{new}}$. If the action is feasible (to be discussed later), $v_{\text{new}}=(\Omega^{v_{\text{new}}},\bbx^{v_{\text{new}}})$ is added to $\ccalV_T$ and the directed edge $(v_{\text{rand}},v_{\text{new}})$ is added to $\ccalE_T$.  
After each successful expansion that produces $v_{\text{new}}$, we test whether the goal is accessible, i.e., if $\texttt{NavigReach}(\bbx^{v_{\text{new}}},\bbx_{\text{goal}},\Omega^{v_{\text{new}}})=\texttt{True}$. Notice that we do not require the goal to be reachable from the root $\bbx_{\text{start}}$ in the final environment but from the current robot position included in $v_{\text{new}}$.

Next, we provide a more detailed description of the sampling and extension steps used by Alg.~\ref{alg:sampling_planner}. Figure \ref{fig:tree-growth} provides the visualization of the tree being built.


\textbf{Sampling a Tree Node:} We sample a tree node $v_{\text{rand}}=[\Omega^{v_{\text{rand}}}, \bbx^{v_{\text{rand}}}]\in\ccalV_T$ according to a discrete mass function $f_\ccalV:\ccalV \rightarrow (0,1)$ (line \ref{alg:sample_node}, Alg.~\ref{alg:sampling_planner}). The mass function can be chosen arbitrarily, provided it assigns nonzero probability to every node (e.g., uniform sampling). 

\textbf{Sampling a Manipulation Action:}
Next, we sample an action $\texttt{Move}(\bbb_i,\bbb_i')$ to apply to the node $v_{\text{rand}}$, generating a new node $v_{\text{new}}$ and the edge $(v_{\text{rand}},v_{\text{new}})$ (line \ref{alg:sample_move_call}, Alg.~\ref{alg:sampling_planner}). Sampling this action occurs hierarchically as follows; see Alg. \ref{alg:sample_move}.


\emph{(i) Sample a high-level triplet:} First, we sample a triplet
\begin{equation}\label{eq:triplet}
s = (b_i,\ccalP_k,g),
\end{equation}
where $b_i \in \ccalB^{\text{m}}$ is a movable block to be relocated, $\ccalP_k \in \ccalP^{\text{nm}}$ is the \emph{placement plane} (i.e., a non-movable plane where $b_i$ will be placed), and $g \in \big(\ccalP^{\text{nm}} \times \ccalP^{\text{nm}}\big)\ \cup\ {\varnothing}$ encodes the intent of this relocation. Specifically, $g=(\ccalP_a,\ccalP_b)$ indicates that $b_i$ will be placed on $\ccalP_k$ to bridge the gap between $\ccalP_a$ and $\ccalP_b$, while $g=\varnothing$ denotes a relocation with no gap-bridging role.

To generate such a triplet $s$, we first construct the set  $\ccalS$, which collects all possible triplets of the form \eqref{eq:triplet}. This set includes:
(i) $s=(b_i,\ccalP_k,(\ccalP_a,\ccalP_b))$ for every $b_i \in \ccalB^{\text{m}}$, $\ccalP_k \in \ccalP^{\text{nm}}$, and non-movable planes $\ccalP_a,\ccalP_b$ such that $\texttt{Gap}(\ccalP_a,\ccalP_b,\Omega(0)) = \texttt{True}$ and $(b_i,\ccalP_k) \in \texttt{Candidate}(\ccalP_a,\ccalP_b)$; and
(ii) $s=(b_i,\ccalP_k,\varnothing)$ for every $b_i \in \ccalB^{\text{m}}$ and all $\ccalP_k \in \ccalP^{\text{nm}}$. The set $\ccalS$ is constructed at the initialization of the tree (line \ref{alg:build_triplets}, Alg.~\ref{alg:sampling_planner}).

We then sample one triplet $s \in \ccalS$ according to a mass function $f_{\ccalS} : \ccalS \to (0,1) $. The function $f_{\ccalS}$ may be chosen arbitrarily, provided it assigns nonzero probability to every $s \in \ccalS$ (e.g., uniform) (line \ref{alg:sample_triplet}, Alg.~\ref{alg:sample_move}).

\begin{figure}[t]
  \centering
  \includegraphics[width=\linewidth]{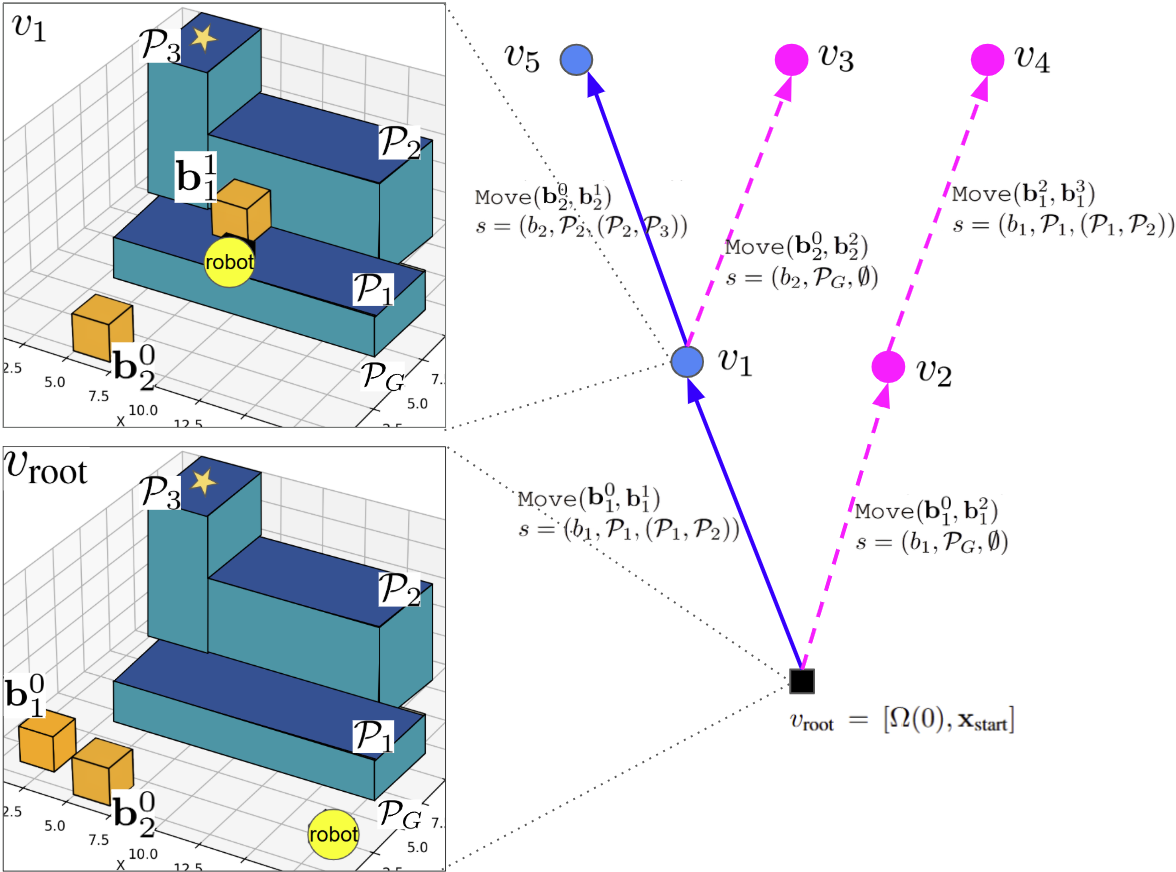}
  \caption{ Graphical illustration of the incremental tree construction; tree nodes are denoted by $v_i$. Each node represents the pair $[\Omega, \mathbf{x}]$, i.e., the joint state of the robot and the blocks, and each edge corresponds to a manipulation action $\texttt{Move}(b_i, b_i')$. A child node therefore differs from its parent in the position of a single block and the robot configuration. The right side of the figure shows the tree built by the proposed planner. \textcolor{black}{Solid blue edges follow the plan $\Pi$ (See Sec. \ref{sec:nonunif1}); the rest of the edges are dashed pink.} Until $v_5$ is discovered, $\mathcal{V}^{\max}_{\Pi}$ contains only $v_1$ since it is the only node realized using $\Pi$. \textcolor{black}{As soon as $v_5$ is added to the tree, $\mathcal{V}^{\max}_{\Pi}$ contains only the node $v_5$.} The left side shows the environment state and robot state at $v_{\text{root}}$ and $v_1$. The environment consists of four non-movable planes: the ground plane $\mathcal{P}_G$ and planes $\mathcal{P}_1$, $\mathcal{P}_2$, and $\mathcal{P}_3$, induced by three non-movable (blue) blocks. The goal (yellow star) lies on $\mathcal{P}_3$. Two movable (orange) blocks, denoted $\mathbf{b}_i^d$, are available, where the superscript $d$ indicates the block’s $d$-th placement.  
  }
  \label{fig:tree-growth}
\end{figure}

(ii) \textit{Sampling the Next Block Configuration:} \phantomsection\label{sec:ii}
The sampled triplet $s$ determines the block $b_i$ used to expand the tree node $v_{\text{rand}}$. Given the intent $g$ and placement plane $\ccalP_k$, we sample a candidate configuration $\bbb_i'$ of $b_i$, defining the action $\texttt{Move}(\bbb_i, \bbb_i')$. Specifically, given $v_{\text{rand}}$ and triplet $s$, we perform up to $N\!\ge\!1$ independent trials to generate $\bbb_i'$. A trial is successful only if the resulting action $\texttt{Move}(\bbb_i, \bbb_i')$ is executable by the robot (as per Defs.~\ref{def:check} and \ref{def:manipreach}) and fulfills the intent $g$. The procedure for sampling $\bbb_i'$ and verifying feasibility consists of the following three steps.

\textit{(ii.1) Sampling robot position and checking feasibility.} \phantomsection\label{sec:ii1}
We first compute the grasp set $\ccalX_{\text{grasp}} = \texttt{ManipReach}(b_i,\Omega^{v'})$, i.e., the set of robot configurations from which the manipulator can grasp $b_i$ in configuration $\bbb_i$. A grasp configuration $\bbx_{\text{grasp}}$ is then sampled according to a distribution $f_{\ccalX_{\text{grasp}}} : \ccalX_{\text{grasp}} \to \mathbb{R}{\ge 0}$ (line \ref{alg:sample_grasp}, Alg.~\ref{alg:sample_move}).
Then, we check whether $\texttt{NavigReach}(\bbx^{v_{\text{rand}}}, \bbx_{\text{grasp}}, \Omega^{v_{\text{rand}}}) = \texttt{True}$, i.e., whether a feasible path exists from the current robot position $\bbx^{v_{\text{rand}}}$ to the grasp pose $\bbx_{\text{grasp}}$ (line \ref{alg:check_grasp}, Alg.~\ref{alg:sample_move}). If this condition fails, the trial is discarded; otherwise, we proceed to the next step.

\textit{(ii.2) Sampling next block configuration and checking intent fulfillment.}
Second, we sample the next configuration $\bbb_i'$ of the block on the placement plane $\ccalP_k$. Let $\ccalC^{v_{\text{rand}}}_{\text{block}}(b_i,\ccalP_k,\Omega^{v_{\text{rand}}})$ denote the configuration space of physically valid poses of $b_i$ where its bottom face lies entirely on $\ccalP_k$ in environment $\Omega^{v_{\text{rand}}}$ (evaluated with $b_i$ treated as carried). We sample $\bbb_i'$ from a distribution $f_{\ccalC^{\,v}_{\text{block}}} : \ccalC^{\,v_{\text{rand}}}_{\text{block}} \rightarrow \mathbb{R}_{\geq 0}$ (line \ref{alg:sample_place}, Alg.~\ref{alg:sample_move}) 
and form the new environment via the defined update:
$\Omega_{\text{temp}}=\Omega^{v_{\text{rand}}}\oplus \texttt{Move}(\bbb_i,\bbb_i')$ (line \ref{alg:update_temp}, Alg.~\ref{alg:sample_move}). Recall that the only difference between $\Omega_{\text{temp}}$ and $\Omega^{v_{\text{rand}}}$ lies in the configuration of $b_i$. Given $\Omega_{\text{temp}}$, we verify whether the  action $\texttt{Move}(\bbb_i,\bbb_i')$ fulfills the intent $g$. Specifically if $g = (\ccalP_a, \ccalP_b)$, we check whether $\texttt{Gap}(\ccalP_a, \ccalP_b, \Omega_{\text{temp}}) = \texttt{False}$, i.e., whether the gap between $\ccalP_a$ and $\ccalP_b$ is now bridged (line \ref{alg:intention_check}, Alg.~\ref{alg:sample_move}). If $g = \varnothing$, the intent is trivially satisfied. If the intent $g$ is not fulfilled, the trial is discarded; otherwise, we proceed to the third step.

\textit{(ii.3) Sampling drop-off position sampling and checking carried-trajectory feasibility.}\phantomsection\label{sec:ii3}
Third, we compute the set $\ccalX_{\text{drop}}$ of all possible robot configurations from which the manipulator can place $b_i$ at $\bbb_i'$, i.e., $\ccalX_{\text{drop}}=\texttt{ManipReach}(b_i,\Omega_{\text{temp}})$. Then, we sample $\bbx_{\text{drop}}$ from a distribution $f_{\ccalX_{\text{drop}}} : \ccalX_{\text{drop}} \rightarrow \mathbb{R}_{\geq 0}$ (line \ref{alg:sample_drop}, Alg.~\ref{alg:sample_move}). We then check if $\texttt{NavigReach}(\bbx_{\text{grasp}},\bbx_{\text{drop}},\Omega_{\text{temp}})=\texttt{True}$, i.e., if there exists a a feasible path from the block grasp location $\bbx_{\text{grasp}}$ to the drop-off position $\bbx_{\text{drop}}$ (line \ref{alg:check_drop}, Alg.~\ref{alg:sample_move}). If this condition met, the current trial is successful and the extension step follows; otherwise, the above process repeats.

\textbf{Tree Extension:} 
As soon as a trial is successful,
we construct the new node $v_{\text{new}}=[\Omega^{v_{\text{new}}},\bbx^{v_{\text{new}}}]$
where $\Omega^{v_{\text{new}}}=\Omega_{\text{temp}}$ and $\bbx^{v_{\text{new}}}=\bbx_{\text{drop}}$ (line \ref{alg:new_node}, Alg.~\ref{alg:sampling_planner}). Then, we update the set of nodes and edges as $\ccalV_T\rightarrow\ccalV_T\cup\{v_{\text{new}}\}$ and $\ccalE_T\rightarrow\ccalE_T\cup\{(v_{\text{rand}},v_{\text{new}})\}$ (line \ref{alg:add_edge}, Alg.~\ref{alg:sampling_planner}). Once an edge $(v_{\text{rand}}, v_{\text{new}})$ is added to the tree, we store the corresponding manipulation action $\texttt{Move}(\bbb_i, \bbb_i')$ along with the robot’s grasp and drop-off positions, $\bbx_{\text{grasp}}$ and $\bbx_{\text{drop}}$, 
that enable its execution. We formalize this by iteratively constructing a function $\ccalA$ that maps each edge in $\ccalE_T$ to its corresponding action and associated robot positions (line \ref{alg:record_action}, Alg.~\ref{alg:sampling_planner}). Next, we check if $\texttt{NavigReach}(\bbx^{v_{\text{new}}},\bbx_{\rm goal},\Omega^{v_{\text{new}}})= \texttt{True}$,
i.e., whether \textcolor{black}{the goal state} is accessible from $\bbx^{v_{\text{new}}}$ in $\Omega^{v_{\text{new}}}$ (line \ref{alg:goal_check}, Alg.~\ref{alg:sampling_planner}). If so, Alg.~\ref{alg:sampling_planner} terminates and returns a plan (which is extracted as described in the next paragraph) (line \ref{alg:backtrack}, Alg.~\ref{alg:sampling_planner}); otherwise, the iteration index is incremented, $\kappa \leftarrow \kappa + 1$, and the process repeats. If all $N$ trials for sampling a valid action $\texttt{Move}(\bbb_i, \bbb_i')$ fail, the iteration index is still incremented, $\kappa \leftarrow \kappa + 1$, without updating the tree.

\textbf{Plan Extraction:}
As soon as the goal test succeeds at $v_{\text{new}}$, we compute the plan $\tau$ solving Problem~\ref{prob_statement} by backtracking from $v_{\text{new}}$ to $v_{\text{root}}$ along $\ccalE_T$ and recovering the corresponding manipulation actions via the function $\ccalA$. Note that $\tau$ specifies the sequence of actions $\tau(k)$ but does not include the robot trajectories to execute them. However, feasible trajectories connecting the grasp and drop-off positions (which can be recovered via $\ccalA$) to execute each action $\tau(k)$ are guaranteed to exist by the feasibility checks performed in steps (ii.1)–(ii.3).


\begin{algorithm}[t]
\caption{\textcolor{black}{\textbf{BRiDGE}:} Sampling‐Based Planner}
\label{alg:sampling_planner}
\SetAlgoLined
\SetKwInOut{Input}{Input}\SetKwInOut{Output}{Output}
\Input{Initial state $(\bbx_{\rm start},\Omega(0))$, goal state $\bbx_{\rm goal}$}
\Output{Sequence of moves $\tau$ reaching $\bbx_{\rm goal}$}
\BlankLine
$\ccalV_T \leftarrow \{\,v_{\text{root}}=[\Omega(0),\bbx_{\rm start}]\,\};\ccalE_T \leftarrow \emptyset; \ccalA \leftarrow \emptyset$\; \nllabel{alg:init_tree}
Build $\ccalS$\; \nllabel{alg:build_triplets}
$\kappa \leftarrow 0$\; \nllabel{alg:iter_init}
\SetKw{KwTo}{do}
\While{$\kappa < K_{\max}$}{{\nllabel{alg:loop}}
$\kappa \leftarrow \kappa + 1$\; \nllabel{alg:iter_incr}
Sample $v_{\text{rand}}=[\Omega^{v_{\text{rand}}},\bbx^{v_{\text{rand}}}]\in\ccalV_T$ via $f_{\ccalV}$\; \nllabel{alg:sample_node}
$[\Omega^{v_{\text{new}}},\bbx_{\rm grasp},\bbx^{v_{\text{new}}},(\bbb_i,\bbb_i'),\texttt{Success}]$ = \texttt{SampleMove}($v_{\text{rand}}$, $\ccalS$)\; \nllabel{alg:sample_move_call}
  \If{\texttt{Success}=\texttt{True}}{
    $v_{\rm new}\leftarrow [\Omega^{v_{\text{new}}},\bbx^{v_{\text{new}}}]$;\ add $v_{\rm new}$ to $\ccalV_T$ and edge $(v_{\text{rand}},v_{\rm new})$ to $\ccalE_T$\; \nllabel{alg:new_node}\nllabel{alg:add_edge}
    $\ccalA[(v_{\text{rand}},v_{\rm new})] \leftarrow (\texttt{Move}(\bbb_i, \bbb_i'),\bbx_{\rm grasp},\bbx^{v_{\text{new}}})$\; \nllabel{alg:record_action}
    \If{\textcolor{black}{$\texttt{NavigReach}(\bbx^{v_{\text{new}}},\bbx_{\rm goal},\Omega^{v_{\text{new}}})$}}{ \nllabel{alg:goal_check}
      \Return backtrack plan $\tau$ using $\ccalA$\; \nllabel{alg:backtrack}
    }
  }
}
 \end{algorithm}

\begin{algorithm}[t]
\caption{\texttt{SampleMove}}
\label{alg:sample_move}
\SetAlgoLined
\SetKwInOut{Input}{Input}\SetKwInOut{Output}{Output}
\Input{Tree node $v_{\text{rand}}$, triplet set $\ccalS$}
\Output{$\Omega_{\text{temp}}$, $\bbx_{\rm grasp}$, $\bbx_{\text{drop}}$,($\bbb_i,\bbb_i'$), Success flag}
\BlankLine
Sample $s=(b_i,\ccalP_k,g)\in\ccalS$ via $f_{\ccalS}$\; \nllabel{alg:sample_triplet}
\For{$\ell=1$ \KwTo $N$}{ \nllabel{alg:inner_trials}
\textcolor{black}{Sample} $\bbx_{\rm grasp}\sim f_{\ccalX_{\rm grasp}}$ on $\ccalX_{\rm grasp}{=}\texttt{ManipReach}(b_i,\Omega^{v_{\text{rand}}})$\; \nllabel{alg:sample_grasp}
\lIf{$\neg\,\texttt{NavigReach}(\bbx^{v_{\text{rand}}},\bbx_{\rm grasp},\Omega^{v_{\text{rand}}})$}{\textbf{continue}} \nllabel{alg:check_grasp}
\textcolor{black}{Sample} $\bbb'_i \sim f_{\ccalC^{\,v}_{\text{block}}}$ on valid poses over $\ccalP_k$\; \nllabel{alg:sample_place}
$\Omega_{\rm temp}\leftarrow \Omega^{v_{\text{rand}}} \oplus \texttt{Move}(\bbb_i,\bbb'_i)$\; \nllabel{alg:update_temp}
\lIf{$(g\neq\emptyset) \textcolor{black}{\wedge} (\texttt{Gap}(\ccalP_a,\ccalP_b,\Omega_{\rm temp}))$}{\textbf{continue}} \nllabel{alg:intention_check}
\textcolor{black}{Sample} $\bbx_{\rm drop}\sim f_{\ccalX_{\rm drop}}$ on $\ccalX_{\rm drop}{=}\texttt{ManipReach}(b_i,\Omega_{\rm temp})$\; \nllabel{alg:sample_drop}
\lIf{$\neg\,\texttt{NavigReach}(\bbx_{\rm grasp},\bbx_{\rm drop},\Omega_{\rm temp})$}{\textbf{continue}} \nllabel{alg:check_drop}
\Return $\Omega_{\text{temp}}, \bbx_{\rm grasp}, \bbx_{\text{drop}},(\bbb_i,\bbb_i'),\texttt{True}$\; \nllabel{alg:mark_success}
}
\Return $\emptyset,\emptyset,\emptyset ,\emptyset,\texttt{False}$\; \nllabel{alg:mark_fail}
\end{algorithm}

\section{\textcolor{black}{Non-Uniform Sampling Strategy}}
\label{sec:biased-sampling}
In this section, we discuss how the mass functions and distributions used to implement the sampling strategy of \textcolor{black}{\textbf{BRiDGE} (Alg. \ref{alg:sampling_planner})} can be constructed. A straightforward choice is to adopt uniform mass functions and distributions for $f_{\ccalV}$, $f_{\ccalS}$, $f_{\ccalX_{\text{grasp}}}$, $f_{\ccalC^{\,v}_{\text{block}}}$, and $f_{\ccalX_{\text{drop}}}$. However, this may slow down the planning process and consequently the construction of feasible plans especially as the number of blocks/planes increases; see Section \ref{sec:Sim}. To accelerate time-to-first-solution, we design non-uniform sampling strategies for (i) node selection \(f_\ccalV\), (ii) triplet selection \(f_{\ccalS}\), and (iii) block-pose sampling \(f_{\ccalC^{\,v}_{\text{block}}}\); the other distributions are selected to be uniform. \textcolor{black}{In Section \ref{sec:nonunif1} we jointly design non-uniform sampling strategies for for \(f_{\ccalV}\) and \(f_{\ccalS}\) while in Section \ref{sec:block_placement_density} we discuss the design of \(f_{\ccalC^{\,v}_{\text{block}}}\). In Section \ref{subsec:fail-refresh}, we discuss how these non-uniform sampling strategies can be refined on-the-fly in case they are biased toward infeasible directions.}

\subsection{Designing \(f_\ccalV\) and \(f_{\ccalS}\)}\label{sec:nonunif1}
\textcolor{black}{Our non-uniform sampling strategy for \(f_\ccalV\) and \(f_{\ccalS}\) is determined by a symbolic plan \(\Pi\)  defined as a sequence of high-level navigation and manipulation actions without validating their feasibility at each step.} 

\subsubsection{BFS Planner (\(\Pi\)):}
\label{subsubsec:bfs-planner}

The plan $\Pi$ will be constructed using a Breadth-First-Search (BFS) planner. This planner can be invoked at any iteration $\kappa$ and tree node \textcolor{black}{denoted by $v_{\Pi}=(\Omega_{\Pi},\bbx_{\pi})\in\ccalV$. Initially, it is invoked at iteration $\kappa=0$ from the root node $v_{\Pi}=v_{\text{root}}$ (i.e., $\Omega_{\Pi}=\Omega_{\text{start}}$ and $\bbx_{\pi}=\bbx_{\text{start}}$) to find a plan $\Pi$ defined as a sequence of triplets, i.e.,  \(\Pi=\langle s_1,\dots,s_M\rangle\) with \(s_m=(b_i,\ccalP_k,g)\),} that if executed the goal region may become accessible to the robot. By executing \(s_m=(b_i,\ccalP_k,g)\) we mean that the robot will place the movable block $b_i$ on the non-movable plane $\ccalP_k$ to fulfill the intent $g$.  
\textcolor{black}{Note that this symbolic plan does not specify the exact placement pose of $b_i$, nor does its existence guarantee that a feasible placement of $b_i$ on $\mathcal{P}_k$ exists to fulfill the intent $g$, or that $\mathcal{P}_k$ is necessarily reachable by the robot.} In fact, all plane/block geometry and robot coordinates are ignored during the construction of $\Pi$; feasibility of geometric execution (reachability, collision-free placements, and path existence) is deferred to Alg.~\ref{alg:sampling_planner}.
The purpose of computing $\Pi$ is to identify promising high-level directions, specifically, an ordering of which blocks should be moved and onto which planes they may be placed, that will be used to construct non-uniform distributions \(f_\ccalV\) and \(f_{\ccalS}\). As a consequence of its symbolic nature, the resulting plan  $\Pi$ may end up being infeasible and non-executable; how the BFS planner replans at later iterations $\kappa>0$ from nodes 
$v_{\Pi}\in\ccalV$ to revise $\Pi$, and therefore refine \(f_\ccalV\) and \(f_{\ccalS}\), will also be discussed later.

In what follows, we discuss how the BFS planner generates the plan $\Pi$ at iteration $\kappa=0$. From the node  $v_{\Pi}$, it searches for a \emph{symbolic} high-level sequence of actions that could lead to a solution to Problem \ref{prob_statement}. Unlike Alg.~\ref{alg:sampling_planner}, which treats each \texttt{Move}$(\bbb_i,\bbb_i')$ as a single high-level (pick–place) action, the BFS planner uses a \emph{finer} action abstraction: at each step it applies exactly one elementary action--\emph{GoTo$\langle\text{Plane}\rangle$}(move to an adjacent plane); \emph{Pick$\langle\text{Block}\rangle$} (pick up a block on the current plane); or \emph{Place$\langle\text{Block}\rangle$} (place the held block on the current plane). This decomposition enables reasoning over long action sequences without validating geometry. 

The BFS planner explores a graph $\ccalG=(\ccalW,\ccalE)$ where $\ccalW$ and $\ccalE$ denote the set of nodes and edges of the graph, respectively. Specifically, $\ccalW$ consists of symbolic states \(w=(\ccalP_{\mathrm{cur}},\,b_h,\,L)\in\ccalW\) where \(\ccalP_{\mathrm{cur}}\in\ccalP^{\mathrm{nm}}\) is the non-movable plane the robot would occupy in that symbolic state; \(b_h\in\ccalB^{\text{m}}\) is a block that is currently carried by the robot or $b_h=\varnothing$ if the robot is not carrying any object; and \(L\) is a function assigning a symbolic status to any block $b$. Formally, \(L:\ccalB^{\text{m}}\to\{\texttt{held}\}\ \cup\ \{(\ccalP_k,g)\mid \ccalP_k\in\ccalP^{\mathrm{nm}},\ g\in\{\emptyset\}\cup\{(\ccalP_a,\ccalP_b)\}\}\). In words, if a block \(b\) rests on \(\ccalP_k\) with intent \(g\), then \(L(b)=(\ccalP_k,g)\); if \(b\) is being carried by the robot (i.e., \(b=b_h\)), then \(L(b)=\texttt{held}\).  
We define the initial node \(w_0=(\ccalP_{\mathrm{start}},\,\emptyset,\,L)\) that is constructed using $v_{\Pi}$; thus, at $\kappa=0$, $\ccalP_{\mathrm{start}}$ is the non-movable plane containing $\bbx_{\Pi}$ in $\Omega_{\Pi}$. To determine the status encoded by $L$ of a block $b$ that is not carried by the robot, it suffices to check whether $b$ is critical for maintaining accessibility between any two non-movable planes. \textcolor{black}{For any pair
$\ccalP_i$ and $\ccalP_j$ that currently satisfies $\texttt{Gap}(\ccalP_i,\ccalP_j,\Omega_{\Pi})=\texttt{False}$, we evaluate whether removing block $b$ from its current plane $\ccalP_k$ results in $\texttt{Gap}(\ccalP_i,\ccalP_j,\Omega_{\Pi-b})=\texttt{True}$, where $\Omega_{\Pi-b}$ is the environment $\Omega_{\Pi}$ with block $b$ completely removed.}  If this occurs, then $b$ is considered critical for ensuring accessibility between $\ccalP_i$ and $\ccalP_j$ and we set $L(b)=(\ccalP_k,(\ccalP_i,\ccalP_j))$;\footnote{In case a block is critical for ensuring accessibility between multiple pairs, then we pick one pair $(\ccalP_i,\ccalP_j)$ randomly.} otherwise, we set $L(b)=(\ccalP_k,\varnothing)$. This process is repeated for all blocks $b$ to construct the mapping function associated with the initial node $w_0$. Note the this function $L$ may change across different states $w\in\ccalW$ as the robot applies the actions \textit{GoTo, Pick, \text{and} Place} introduced earlier as it will be discussed next.

The set of edges $(w,w')\in\ccalE$ is induced by the above-mentioned elementary actions (\textit{GoTo, Pick, Place}). Next, we describe in detail the actions available at a symbolic state \(w=(\ccalP_{\mathrm{cur}},b_h,L)\) and the corresponding next state $w'$. At any state $w$ the following elementary actions are \emph{feasible}:

\begin{itemize}
  \item \emph{GoTo}(\(\ccalP'\)). This action instructs the robot to move to plane $\ccalP'$ from the current plane $\ccalP_{\text{cur}}$. This action  can be applied at node $w$ only if $\ccalP'$ is reachable from $\ccalP_{\text{cur}}$. We say that this is possible if iff (i) \(\texttt{Gap}(\ccalP_{\text{cur}},\ccalP',\hat{\Omega}(0))=\texttt{False}\) or (ii) there exists some block \(b\) with \(L(b)=(\ccalP_{k},(\ccalP_{\text{cur}},\ccalP'))\). In (i), $\hat{\Omega}(0)$ stands for the environment with all movable blocks completely removed. Thus if (i) is satisfied, it means that $\ccalP_{\text{cur}}$ and $\ccalP'$ are pairwise accessible without the `help' of a movable block \textcolor{black}{(i.e. by structure they are accessible)}. The condition in (ii) implies that the considered planes are connected through the movable block $b$. Given \(w\) and \emph{GoTo}\((\ccalP')\), the successor is \(w'=(\ccalP',\,b_h,\,L)\).
  %
  \item \emph{Pick}(\(b\)). If no block is being carried in the symbolic state (\(b_h=\emptyset\)), then any block on the current plane $\ccalP_{\mathrm{cur}}$ may be picked, i.e., for each \(b\) with \(L(b)=(\ccalP_{\mathrm{cur}},g)\) and for some \(g\), the action \emph{Pick}\((b)\) can be selected. Given \(w\) and \emph{Pick}\((b)\), the successor is \(w'=(\ccalP_{\mathrm{cur}},\,b,\,L)\) with $L(b)=\texttt{held}$.
  \item \emph{Place}(\(b_h,\ccalP_{\mathrm{cur}},g\)). If a block is being carried (\(b_h\neq\emptyset\)), it may be placed on the current plane either temporarily (\(g=\emptyset\)) or with any admissible bridging intent \(g=(\ccalP_a,\ccalP_b)\) such that \((b_h,\ccalP_{\mathrm{cur}})\in\texttt{Candidate}(\ccalP_a,\ccalP_b)\). Given \(w\) and \emph{Place}\((b_h,\ccalP_{\mathrm{cur}},g)\), the successor is \(w'=(\ccalP_{\mathrm{cur}},\,\emptyset,\,L)\) with $L(b)=(\ccalP_{\mathrm{cur}},g)$.
\end{itemize}

The BFS algorithm is applied is to find a path over $\ccalG$ connecting the initial node \(w_0=(\ccalP_{\mathrm{start}},\,\emptyset,\,L)\) to a goal node \(w_g=(\ccalP_{\mathrm{cur}},b_h,L)\) satisfying \(\ccalP_{\mathrm{cur}}=\ccalP_{\mathrm{goal}}\), where \(\ccalP_{\mathrm{goal}}\) is the non-movable plane containing \(\bbx_{\rm goal}\). 
%
Let \(w_0,\dots,w_g\) be the shortest node path in \(\ccalG\). This path corresponds to a sequence of \emph{GoTo}/\emph{Pick}/\emph{Place} actions; however, we export, in order, only the \emph{Place}(\(b_i,\ccalP_k,g\)) actions to obtain \(\Pi=\langle s_1,\dots,s_M\rangle\) with \(s_m=(b_i,\ccalP_k,g)\). 

\subsubsection{Non-uniform Mass Function \(f_{\ccalV}\):}
\label{subsubsec:plan-frontier}
In what follows we design a non-uniform mass function $f_{\ccalV}$ guided by the plan $\Pi$. To construct $f_{\ccalV}$, we need first to introduce the function $\mathrm{PlanStep}:\ccalV\rightarrow\mathbb{N}$. This function returns \(\mathrm{PlanStep}(v)=m\) for any node \(v\in\ccalV_T\), \emph{iff} the path from $v_{\Pi}$ 
to \(v\) has exactly \(m\) edges and those edges, in order, realize the steps \(s_1,\dots,s_m\) of \(\Pi\). That is, the sequence of triplets used to sample the move actions that enable reaching \(v\) from \(v_{\text{root}}\) is exactly the same as those in $\Pi$, matching block, placement plane, and intent $g$.
%
Otherwise we set \(\mathrm{PlanStep}(v)=\varnothing\).

We also need to introduce the set $\ccalV^{\max}_{\Pi}$ defined as follows.
\[
\textcolor{black}{\ccalV^{\max}_{\Pi}=\big\{\,v\in\ccalV_T\;:\;\mathrm{PlanStep}(v)=\max_{v'\in\ccalV}\mathrm{PlanStep}(v')\big\}}.
\]
Intuitively, \(\ccalV^{\max}_{\Pi}\) contains those nodes that have progressed \emph{furthest along} \(\Pi\) (see also Fig. \ref{fig:tree-growth}). 

Given $\ccalV^{\max}_{\Pi}$, we define the following non-uniform mass function:
\begin{equation}\label{eq:fv_define}
    f_{\ccalV}(v)=
\begin{cases}
\dfrac{p_{\text{plan}}}{|\ccalV^{\max}_{\Pi}|}, & v\in \ccalV^{\max}_{\Pi},\\[6pt]
\dfrac{1-p_{\text{plan}}}{\,|\ccalV_T\setminus \ccalV^{\max}_{\Pi}|\,}, & v\in \ccalV_T\setminus \ccalV^{\max}_{\Pi},
\end{cases}
\end{equation}
for some $p_{\text{plan}}\in(0,1)$. In words, with probability $p_{\text{plan}}\in(0,1)$ we sample uniformly from \(\ccalV^{\max}_{\Pi}\), and with probability $1-p_{\text{plan}}$ we sample uniformly from $\ccalV_T\setminus\ccalV^{\max}_{\Pi}$. Observe that the higher the value for $p_{\text{plan}}$ is, the more biased $f_{\ccalV}$ gets toward selecting nodes $v\in\ccalV^{\max}_{\Pi}$. 
Uniform sampling can be recovered by choosing
\(p_{\text{plan}}=\frac{|\ccalV^{\max}_{\Pi}|}{|\ccalV_T|}\),
which makes \(f_{\ccalV}\) constant over \(\ccalV_T\).



\subsubsection{Non-uniform Mass \(f_{\ccalS}\):}
\label{subsubsec:triplet-selection}
Given the chosen node \(v=v_{\text{rand}}\), we partition the set \(\ccalS\) of all triplets into three buckets.
\textbf{(i)} \emph{the next triplet recommended by the symbolic plan \(\Pi\)}: If \(\texttt{PlanStep}(v)=m\) and the next symbolic step in $\Pi$ is \(s_{m+1}=(b_i,\ccalP_k,g)\), 
set \(\ccalS_{\text{plan}}(v)=\{\,s_{m+1}\,\}\); otherwise if \(\mathrm{PlanStep}(v)=\varnothing\), then \(\ccalS_{\text{plan}}(v)=\emptyset\).
\textbf{(ii)} \emph{triplets that bridge a gap}:
\(\ccalS_{\text{gap}}(v)=\{\, (b_i,\ccalP_k,(\ccalP_a,\ccalP_b)) \in \ccalS \,\}\).
\textbf{(iii)} \emph{triplets that do not bridge a gap}:
\(\ccalS_{\text{nogap}}(v)=\{\, (b_i,\ccalP_k,\emptyset) \in \ccalS \,\}\).
Note that \(\ccalS=\ccalS_{\text{plan}}(v)\ \cup\ \ccalS_{\text{gap}}(v)\ \cup\ \ccalS_{\text{nogap}}(v)\).

Let \(p_1,p_2,p_3>0\) with \(p_1{+}p_2{+}p_3=1\).
Given the node $v$ the triplet is sampled according to the following mass function
\begin{equation}\label{eq:fs_define}
f_{\ccalS}(s\mid v)=
\begin{cases}
p_1, &  s\in \ccalS_{\text{plan}}(v),\\
\frac{p_2}{|\ccalS_{\text{gap}}(v)|}, & s\in \ccalS_{\text{gap}}(v)\ \\
\frac{p_3}{|\ccalS_{\text{nogap}}(v)|}, & s\in \ccalS_{\text{nogap}}(v)\ 
\end{cases}
\end{equation}
Note that if \(\ccalS_{\text{plan}}(v)=\emptyset\), 
then we redistribute the probability mass by setting \(p_1=0\) and renormalizing over the remaining buckets so that \(p_2,p_3>0,\) and \(p_2+p_3=1\).
Observe that if \(\ccalS_{\text{plan}}(v)\neq\emptyset\), then the higher the value for $p_1$, the more $f_{\ccalS}$ is biased toward selecting triplets that follow the plan $\Pi$. 
Uniform triplet sampling can be recovered by setting
\(p_1 = |S_{\text{plan}}(v)|/|\mathcal S|\), \(p_2 = |S_{\text{gap}}(v)|/|\mathcal S|\), \(p_3 = |S_{\text{nogap}}(v)|/|\mathcal S|\).

\subsection{\textcolor{black}{Designing} \(f_{\ccalC^{\,v}_{\text{block}}}\)}\label{sec:block_placement_density}
\label{subsubsec:sampling-kernels}
In this section, we design a non-uniform distribution \(f_{\ccalC^{\,v}_{\text{block}}}\) that generates the drop-off location of a \textcolor{black}{selected block $b_i$} on a given plane $\ccalP_k$.
Let \(\ccalC^{\,v}_{\text{block}}(b_i,\ccalP_k,\Omega^{\,v})\) be a set collecting all collision-free poses of \(b_i\) in the free space of \(\ccalP_k\) given the environment $\Omega^{\,v}$ associated with the tree node $v$. 
For brevity, let \(\ccalC^{\,v}_{\text{block}}\equiv \ccalC^{\,v}_{\text{block}}(b_i,\ccalP_k;\Omega^{\,v})\).
The distribution \(f_{\ccalC^{\,v}_{\text{block}}}(\bbb_i'\mid b_i,\ccalP_k,g,\Omega)\) depends on the intent $g$ of the \textcolor{black}{selected block $b_i$}. Specifically, we consider the following two cases:

\textcolor{black}{\textbf{Temporary place} (\(g=\varnothing\)): If $g=\varnothing$, i.e., there is no intent to bridge a gap, then the  block \(b_i\) may be placed anywhere on \(\ccalP_k\). Thus, we design \(f_{\ccalC^{\,v}_{\text{block}}}\) to be a uniform distribution over $\ccalC^{\,v}_{\text{block}}$, i.e.,\(f_{\ccalC^{\,v}_{\text{block}}}(\bbb_i' \mid b_i,\ccalP_k,g=\emptyset,\Omega)=\mathrm{Uniform}\!\big(\ccalC^{\,v}_{\text{block}}\big)\).}

\begin{figure}[t]
    \centering
    \subfigure[]{\label{fig:delta1}\includegraphics[width=0.33\textwidth]{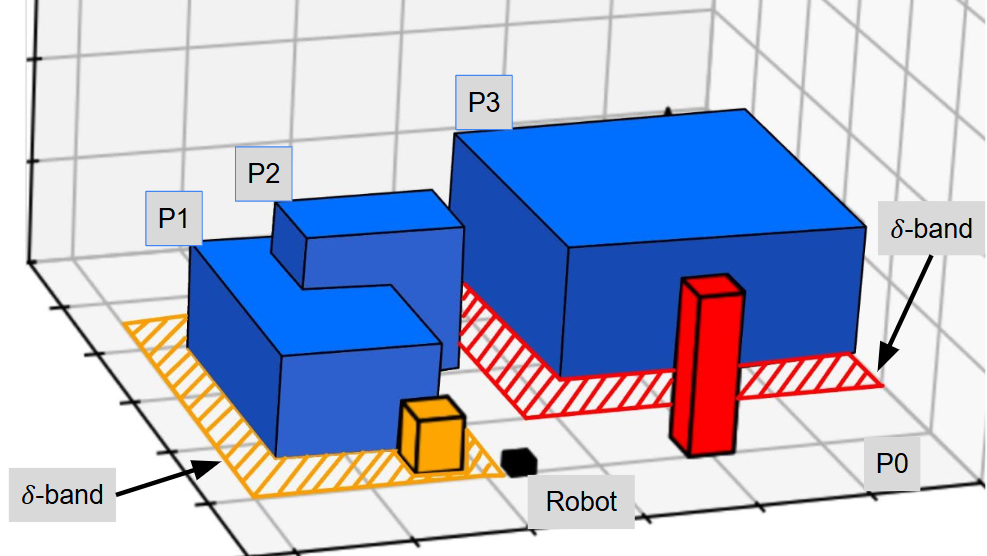}} \hfill
    \subfigure[]{\label{fig:delta2}\includegraphics[width=0.33\textwidth]{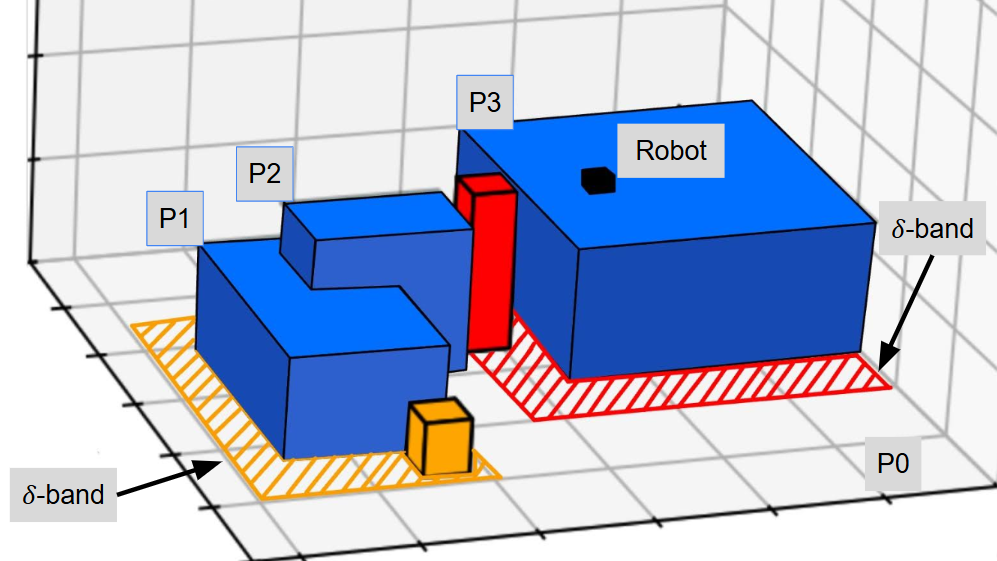}} \hfill
    \caption{Example illustrating $\delta$-bands for two independent gaps in the environment: $(\ccalP_G,\ccalP_1)$ and $(\ccalP_2,\ccalP_3)$. In both Fig. \ref{fig:delta1} and Fig. \ref{fig:delta2}, the yellow hatched region shows the $\delta$-band for the gap $(\ccalP_G,\ccalP_1)$; here the placement plane for the yellow block is $\ccalP_k = \ccalP_G$, which coincides with one of the gap planes, and thus the boundary of $\ccalP_1$ is projected onto $\ccalP_G$ and dilated to obtain the yellow $\delta$-band. The red hatched region shows the $\delta$-band for the gap $(\ccalP_2,\ccalP_3)$; in this case, the placement plane $\ccalP_k = \ccalP_G$ for the red block is neither $\ccalP_2$ nor $\ccalP_3$, and so the boundary of $\ccalP_3$ (even $\ccalP_2$ could be chosen) is projected onto $\ccalP_G$ and dilated to form the red $\delta$-band. Fig. \ref{fig:delta2} shows the red block in a pose within its corresponding $\delta$-band to bridge the gap.}
    \label{fig:delta_band}
\end{figure}

\textbf{Bridge} ($g=(\ccalP_ i,\ccalP_j)$): If $g=(\ccalP_i,\ccalP_j)$, this means that $b_i$ should be placed on $\ccalP_k$ to bridge the gap between two planes $\ccalP_i$ and $\ccalP_j$. \textcolor{black}{Thus, in this case, we design a distribution that biases placements of $b_i$ on $\mathcal{P}_k$ toward a $\delta$-wide band around the orthogonal projection of the relevant supporting plane. Specifically, if $\mathcal{P}_i \neq \mathcal{P}_k$ and $\mathcal{P}_j \neq \mathcal{P}_k$, we project either $\mathcal{P}_i$ or $\mathcal{P}_j$ onto $\mathcal{P}_k$ (see Fig. \ref{fig:case1b}–\ref{fig:case1c}; see also Fig. \ref{fig:delta_band}). If $\mathcal{P}_k = \mathcal{P}_j$, we project $\mathcal{P}_i$ onto $\mathcal{P}_k$. In both cases, the resulting projection is expanded by $\delta$ to form the $\delta$-band (see Fig. \ref{fig:case1b}–\ref{fig:case1c}, where the boundary of Plane 2 is projected onto Plane 1 and then dilated to obtain the $\delta$-band; see also Fig. \ref{fig:delta_band}).}
Let $\ccalC_{\text{gap}}^{\textcolor{black}{\delta}}\subseteq \ccalC^{\,v}_{\text{block}}$ be a set collecting all \textcolor{black}{collision-free poses of $b_i$} 
inside this $\delta$-wide band. Then, we define \(f_{\ccalC^{\,v}_{\text{block}}}\) as follows:
\begin{equation}\label{eq:fc_define}
\begin{aligned}
f_{\ccalC^{\,v}_{\text{block}}}(\bbb_i'\mid b_i,\ccalP_k,&(\ccalP_i,\ccalP_j),\Omega)
= p_{\text{gap}}\cdot \mathrm{Uniform}(\ccalC_{\text{gap}}^{\delta})\\
&+(1-p_{\text{gap}})\cdot \mathrm{Uniform}\!\big(\ccalC^{\,v}_{\text{block}}\setminus \ccalC_{\text{gap}}^{\delta}\big),
\end{aligned}
\end{equation}
with $0<p_{\text{gap}}<1$.
This biases \(\bbb_i'\) toward regions \textcolor{black}{(captured by $\ccalC_{\text{gap}}^{\delta}$)}  that are most likely to bridge the gap while ensuring every collision-free pose in $\ccalC^{\,v}_{\text{block}}$ retains nonzero probability. 
\textcolor{black}{Uniform sampling can be recovered by setting $\delta=0$ which results in $\ccalC_{\text{gap}}^{\delta}=\emptyset$.} 

\subsection{\textcolor{black}{Revising $\Pi$ Due To Infeasible Recommendations}}
\label{subsec:fail-refresh}


\textcolor{black}{As discussed earlier, the symbolic plan $\Pi$ may contain triplets that cannot be executed. 
Recall that a triplet $s_m = (b_i, \mathcal{P}_k, g)$ indicates that the block $b_i$ should be placed on plane $\mathcal{P}_k$ to fulfill the intent $g$. Such a triplet $s_m$ contained in the plan $\Pi$ may be infeasible/non-executable for the following reasons. First, construction of the plan $\Pi$ relies on the \texttt{Candidate} function, introduced in Definition \ref{def:Candidate}, which may generate blocks that may not necessarily be able to bridge desired gaps. Thus, if $g$ in $s_m$ requires bridging a gap, it is possible that $b_i$ may not be able to do so. Recall that construction of $\Pi$ does not include any low-level geometric feasibility tests, as it simply aims to construct fast a high-level sequence of triplets $s_m$ that may result in task completion. 
Second, regardless of the intent $g$ of block $b_i$ on $\mathcal{P}_k$, it is also possible that $\mathcal{P}_k$ is not reachable to the robot despite feasibility checks that occur before applying the \textit{GoTo} action during construction of $\Pi$. This may occur since the feasibility tests associated with the \textit{GoTo} action are not sufficient. For example, if condition (i) of \textit{GoTo} is satisfied (i.e., two planes are connected by structure when no movable blocks exist in the environment), these planes may end up being disconnected in an environment associated with a tree node $v$ due to the presence of other movable blocks. We emphasize again that such low-level geometric feasibility tests are deferred to Alg.~\ref{alg:sample_move} (see lines \ref{alg:check_grasp}, \ref{alg:intention_check}, and \ref{alg:check_drop}), as Alg.~\ref{alg:sample_move} is responsible for generating feasible robot poses and drop-off locations of blocks fulfilling sampled intents $g$.
If Alg.~\ref{alg:sample_move} fails to validate feasibility of a triplet $s_m$ (within $N$ attempts) that is contained in the plan $\Pi$, then we demote $s_m$. Specifically, we remove it from $\mathcal{S}_{\text{plan}}(v)$ and keep it in $\mathcal{S}_{\text{gap}}(v)$ (with no special weight at that node). In this case, $f_{\mathcal{S}}$ is not driven by $\Pi$ anymore by construction.}

\textcolor{black}{As soon as this occurs, in the next outer iteration $(\kappa{+}1)$, after sampling a node $v_{\text{rand}} \sim f_\mathcal{V}$, we recompute a new symbolic plan $\Pi'$ exactly as before, setting $v_{\Pi} = v_{\text{rand}}$ (any node in $\mathcal{V}$ may be chosen). During this recomputation, we explicitly exclude the symbolic action $Place(s_m)$ from the BFS process to prevent generating a new plan $\Pi'$ that again recommends the triplet $s_m$ that Alg.~\ref{alg:sample_move} failed to validate. The resulting plan $\Pi'$ then replaces $\Pi$ as the symbolic plan guiding the non-uniform sampling strategy.  Notice that since $\Pi'$ is invoked at a node $v_{\Pi} \neq v_{\text{root}}$, we have $\textit{PlanStep}(v)=0$ for all nodes (e.g., $v_{\text{root}}$) that are not reachable from $v_{\Pi}$ in the incrementally constructed directed tree $\mathcal{T}$. We repeat this process for each symbolic action deemed infeasible. Note that although $s_m$ is removed from the symbolic plan, it can still be sampled by Alg.~\ref{alg:sampling_planner} due to its presence in $\mathcal{S}_{\text{gap}}(v)$.}

\begin{rem}[High-Level Planner Generating $\Pi$]\label{rem:LLM}
    \textcolor{black}{The BFS planner used to compute the plan $\Pi$ over the graph $\mathcal{G}$ may be replaced by other planners such as A*. Moreover, because $\Pi$ need not be feasible, it can also be generated using heuristic planners, such as Large Language Models (LLMs), that are more user-friendly. In this case, an additional source of infeasibility arises from the fact that such planners are not guaranteed to be sound. In Section \ref{sec:Sim}, we show how LLMs can be prompted to construct $\Pi$ and discuss the associated trade-offs.}
\end{rem}

\section{Probabilistic Completeness}
\label{sec:proofGuarantees}

In this Section, we show that our algorithm is \emph{probabilistically complete}, \textcolor{black}{i.e., that the probability that our algorithm will find a feasible plan $\tau$ (if it exists), addressing Problem \ref{prob_statement}, goes to $1$ as the number of iterations $\kappa$ of Alg. \ref{alg:sampling_planner} goes to infinity.} To formally state this result, we first introduce the notion of \emph{plan clearance} \textcolor{black}{(which is also used in \cite{van2010path,zhang2025namo})}. Consider any plan
\[
\tau \;=\; \langle \tau(1),\dots,\tau(k),\dots,\tau(H)\rangle
\]
returned by Algorithm~\ref{alg:sampling_planner} that solves Problem~\ref{prob_statement}, where each action is of the form
\[
\tau(k) \;=\; \texttt{Move}\big(\,\bbb_{i_k},\,\bbb_{i_k}^{\star}\big)
\]
moving block \(b_{i_k}\) to configuration \(\bbb_{i_k}^{\star}\).
We say that \(\tau\) has \emph{clearance} \(\varepsilon>0\) if, for any alternative plan
\[
\tilde\tau \;=\; \langle \tilde\tau(1),\dots,\tilde\tau(k),\dots,\tilde\tau(H)\rangle
\]
satisfying, for all \(k\in\{1,\dots,H\}\), the following two conditions:\\
(i) \(\tilde\tau(k)=\texttt{Move}\big(\bbb_{i_k},\,\tilde\bbb_{i_k}\big)\) moves the \emph{same block} \(b_{i_k}\) as \(\tau(k)\) (the block identity at step \(k\) is unchanged);\\
(ii)\(\big\|\,\tilde\bbb_{i_k}-\bbb_{i_k}^{\star}\,\big\| \;<\; \varepsilon\) with respect to a fixed norm on the block-pose space,\\
the plan \(\tilde\tau\) is also feasible. In what follows, we concisely denote the satisfaction of (i)–(ii) at step \(k\) by
$\|\tau(k)-\tilde\tau(k)\| \;\le\; \varepsilon$.

\begin{theorem}[Probabilistic completeness]\label{thm:pc}
Assume there exists a feasible plan 
\(\tau=\langle \tau(1),\dots,\tau(H)\rangle\) 
with clearance \(\varepsilon>0\) (as defined above)\textcolor{black}{, and that Assumption~\ref{ass:cand} holds.}
Let the sampling distributions
\(f_{\mathcal V}\), \(f_{\mathcal S_{\mathrm{Cand}}}\), and \(f_{\mathcal C^{\,v}_{\mathrm{block}}}\)
be those defined in Sec. \ref{sec:biased-sampling}, and let
\(p_{\mathrm{plan}},\,p_1,\,p_2,\,p_3,\,p_{\mathrm{gap}}\in (0,1)\).
Then the probability that Algorithms~\ref{alg:sampling_planner}--\ref{alg:sample_move} return a feasible plan converges to \(1\) as the number of iterations \(\kappa\to\infty\).
\end{theorem}

\begin{proof}
To show this result, we first introduce the following notations and definitions.
Let \(\ccalT^{\kappa}=\{\ccalV_T^{\kappa},\ccalE_T^{\kappa}\}\) denote the tree constructed by Alg.~\ref{alg:sampling_planner} at the end of iteration \(\kappa\).
Since the plan \(\tau\) has clearance \(\varepsilon>0\), there exist alternative feasible plans \(\tilde\tau\) such that \(\|\tau(k)-\tilde\tau(k)\|\le \varepsilon\) for all \(k\).
Throughout the proof, any reference to \(\tilde\tau(k)\) assumes this \(\varepsilon\)-closeness to \(\tau\).

We define two random variables used to model the sampling process and the computation of a feasible plan.
First, we define Bernoulli random variables \(X_{\kappa}\) that are equal to \(1\) if the following holds for some \(k\in\{1,\dots,H\}\) at iteration \(\kappa\):
(i) a node \(v\) is added to the tree via an edge \((u,v)\) that \textcolor{black}{\emph{realizes step \(k\) of \(\tilde\tau\)}, that is, the edge executes the action $\tilde\tau(k) = \texttt{Move}(\bbb_{i_k},\tilde\bbb_{i_k})$ with \(\|\tilde\bbb_{i_k}-\bbb_{i_k}^{\star}\|\le \varepsilon\)}, and its predecessor node \(u\) corresponds to the realization of step \(k{-}1\) (for \(k=1\), take \(u=v_{\mathrm{root}}\)); and
(ii) no node $v'$ realizing step \(k\) existed in the tree prior to iteration \(\kappa\), i.e., there was no edge ($u',v'$) that was $\varepsilon$ close to $\tau_k$.
Additionally, once a node corresponding to step \(H\) that satisfies both (i) and (ii) is added at some iteration \(\kappa\), no future iterations can satisfy condition (ii) for any \(k\); in this case, by convention, we set \(X_{\kappa'}=1\) for all \(\kappa'>\kappa\).

Second, we define the random variables \(Y_{\kappa}\) that count the number of successes among the Bernoulli variables \(X_1,\dots,X_{\kappa}\); that is,
\(Y_{\kappa}=\sum_{i=1}^{\kappa} X_i\).
By construction of \(X_i\) (and the above convention), \(Y_{\kappa}\ge H\) \emph{iff} the node corresponding to the action \(\tilde\tau(H)\) has been added to the tree \(\ccalV_T^{\kappa}\), i.e., a solution has been found.
Therefore, to prove probabilistic completeness it suffices to show that \(\mathbb{P}(Y_{\kappa}\ge H)\to 1\) as \(\kappa\to\infty\).

Using Chebyshev’s inequality, we bound the failure probability (fewer than \(H\) successful expansions in \(\kappa\) iterations, i.e., the goal is not reached) as follows:
\begin{equation}
\begin{aligned}
\mathbb{P}\!\left(Y_{\kappa}<H\right)
&= \mathbb{P}\!\left(\mathbb{E}[Y_{\kappa}]-Y_{\kappa} \;>\; \mathbb{E}[Y_{\kappa}]-H\right) \\
&\le \mathbb{P}\!\left(\,\lvert Y_{\kappa}-\mathbb{E}[Y_{\kappa}]\rvert \;\ge\; \mathbb{E}[Y_{\kappa}]-H\,\right) \\
&\le \frac{\operatorname{Var}(Y_{\kappa})}{\big(\mathbb{E}[Y_{\kappa}]-H\big)^2}\,,
\end{aligned}
\label{eq:chebyshev-bound}
\end{equation}
where the second inequality is Chebyshev’s bound. We next express \(\mathbb{E}[Y_{\kappa}]\) and \(\operatorname{Var}(Y_{\kappa})\) in terms of the Bernoulli variables \(X_i\). Let \(p_i := \mathbb{P}(X_i=1)\). Then
\begin{align}
\mathbb{E}[Y_{\kappa}]
&= \sum_{i=1}^{\kappa} \mathbb{E}[X_i]
 = \sum_{i=1}^{\kappa} p_i,
\label{eq:EYkappa}
\\[4pt]
\operatorname{Var}(Y_{\kappa})
&= \sum_{i=1}^{\kappa} \operatorname{Var}(X_i)
   + 2 \sum_{1\le i<j\le \kappa} \operatorname{Cov}(X_i,X_j)
\\
&= \sum_{i=1}^{\kappa} p_i\,(1-p_i)
   + 2 \sum_{1\le i<j\le \kappa} \Big(\mathbb{E}[X_i X_j] - p_i p_j\Big).
\label{eq:VarYkappa-expanded}
\end{align}

In what follows we show:
\textbf{(a)} $\lim_{\kappa\to\infty} \mathbb{E}[Y_{\kappa}] = \infty$,
and \textbf{(b)} $\operatorname{Var}(Y_{\kappa}) \;\le\; 3\,\mathbb{E}[Y_{\kappa}] \;\;\text{for all }\kappa$.
Using \textbf{(b)}, we can rewrite the bound in~\eqref{eq:chebyshev-bound} as
\begin{equation}
\mathbb{P}\!\left(Y_{\kappa}<H\right)
\;\le\; \frac{\operatorname{Var}(Y_{\kappa})}{\big(\mathbb{E}[Y_{\kappa}]-H\big)^2}
\;\le\; \frac{3\,\mathbb{E}[Y_{\kappa}]}{\big(\mathbb{E}[Y_{\kappa}]-H\big)^2}\,.
\label{eq:chebyshev-3Ey}
\end{equation}
Combining \textbf{(a)} with~\eqref{eq:chebyshev-3Ey} yields
\[
\lim_{\kappa\to\infty}\mathbb{P}\!\left(Y_{\kappa}<H\right)=0.
\]
Equivalently,
\[
\lim_{\kappa\to\infty}\mathbb{P}\!\left(Y_{\kappa}\ge H\right)=1,
\]
which completes the proof of probabilistic completeness.

\paragraph{\textbf{Proof of (a).}}
We show that \(\lim_{\kappa\to\infty}\mathbb{E}[Y_{\kappa}]=\infty\).
Let \(p_{\kappa}:=\mathbb{P}(X_{\kappa}=1)\).
Since \(\mathbb{E}[Y_{\kappa}]=\sum_{i=1}^{\kappa}p_i\), it suffices to establish a divergent lower bound on \(\sum_{i=1}^{\kappa}p_i\).

Fix a step index \(k\in\{1,\dots,H\}\) that has not yet been sampled.
Let \(v_{k-1}\) denote a parent node in the tree from which step \(k\) can be realized (for \(k=1\), \(v_0=v_{\mathrm{root}}\)).
Let \(s_k^\star\) be the triplet that enables step \(\tilde\tau(k)\) at \(v_{k-1}\).
During iteration \(\kappa\), the event \(\{X_\kappa=1\}\) occurs if we:
(i) select the suitable predecessor node \(v_{k-1}\);
(ii) select the triplet \(s_k^\star\); and
(iii) succeed in sampling block configuration (\hyperref[sec:ii1]{ii.1}-\hyperref[sec:ii3]{ii.3}) so that the resulting action satisfies \(\|\tilde\tau(k)-\tau(k)\|\le \varepsilon\).
We lower bound each stage as follows.

\emph{Node choice.}
Using the node distribution \(f_{\ccalV}\), every node has nonzero mass (since \(p_{\mathrm{plan}}\in(0,1)\) and sampling within buckets is uniform).
Algorithm~\ref{alg:sampling_planner} adds at most one node per iteration, hence \(|\ccalV_T^{\kappa}|\le \kappa+1\). Therefore by construction of $f_\ccalV$ in [\ref{eq:fv_define}]
\begin{equation}
\begin{aligned}
\underbrace{f_\ccalV\!\big(v_{k-1}\,\big|\,\kappa\big)}_{=:~p_V(\kappa)}
&\;\ge\; \frac{\min\{\,p_{\mathrm{plan}},\,1-p_{\mathrm{plan}}\,\}}{\kappa+1}
\\
&\;=\; \frac{c_V}{\kappa+1},
\qquad c_V\in(0,1).
\end{aligned}
\label{eq:pv-lb-rewrite}
\end{equation}
where $c_V := \min\{p_{\text{plan}},\,1-p_{\text{plan}}\}$

\emph{Triplet choice.}
At node \(v_{k-1}\), the candidate set \(\ccalS\) is finite. 
By construction of \(f_{\ccalS}\) in [\ref{eq:fs_define}] (bucket masses \(p_1,p_2,p_3\in(0,1)\) and uniform within buckets), we get
\begin{equation}
\begin{aligned}
\underbrace{f_{\ccalS}\!\big(s_k^\star\,\big|\,v_{k-1}\big)}_{=:~p_S(k)}
&\;\ge\;
\min\!\left\{\,p_1,\frac{p_2}{|\ccalS_{\mathrm{gap}}(v_{k-1})|},\frac{p_3}{|\ccalS_{\mathrm{nogap}}(v_{k-1})|}\right\}
\\
&\;=\; c_S(k)\;\in(0,1).
\end{aligned}
\label{eq:ps-lb-rewrite}
\end{equation}
where $c_S(k)=\min\!\left\{\,p_1,\frac{p_2}{|\ccalS_{\mathrm{gap}}(v_{k-1})|},\frac{p_3}{|\ccalS_{\mathrm{nogap}}(v_{k-1})|}\right\}$. Observe that $c_S(k)>0$, since $p_1,p_2,p_3>0$ \textcolor{black}{and because Assumption~\ref{ass:cand} ensures \(s_k^\star\in\ccalS\); otherwise \(c_S(k)\) could be zero}.
Since \(k\in\{1,\dots,H\}\) ranges over a finite set, we can fix a uniform step–independent lower bound
\begin{equation}
c_S\;:=\;\min_{1\le k\le H} c_S(k)\ \in (0,1).
\label{eq:cS-uniform}
\end{equation}

\emph{Block–configuration choice.}
Let \(\mathbb{B}_k:=\{\bbb:\ \|\bbb-\bbb_{i_k}^{\star}\|<\varepsilon\}\) denote the \(\varepsilon\)-ball around the target configuration given by $\tau(k)$.
Since by construction of \(f_{\mathcal C^{\,v}_{\mathrm{block}}}\) [\ref{eq:fc_define}], the distribution is uniform over two feasible subsets of feasible block configurations with weights \(p_{\mathrm{gap}}\in(0,1)\) and \(1-p_{\mathrm{gap}}\); and since \(\varepsilon>0\), we obtain a strictly positive single–draw success probability $p_\mathbb{B}(k)>0$ which is the probability of choosing a block configuration $\tilde\bbb_{i_k}$ inside the $\varepsilon$-ball $\mathbb{B}_k$.
With \(N\ge 1\) i.i.d.\ attempts per iteration, the probability of at least one success is
\begin{equation}
p_{\mathbb{B},N}(k)\;:=\;1-\big(1-p_\mathbb{B}(k)\big)^{N}\ \in (0,1).
\label{eq:pGN-rewrite}
\end{equation}
Since \(H\) is finite, we can define the step–uniform constant
\begin{equation}
c_{\mathbb{B},N}\;:=\;\min_{1\le k\le H}\,p_{\mathbb{B},N}(k)\ \in (0,1).
\label{eq:cGN-rewrite}
\end{equation}

\emph{Per–iteration success lower bound.}
Combining the above three stages, we get
\begin{equation}
\begin{aligned}
p_{\kappa}
&\;=\;\mathbb{P}(X_{\kappa}=1)
\\
&\;\ge p_V(\kappa)\cdot p_S \cdot p_{\mathbb{B},N}
\\
&\;=\; \frac{c_V}{\kappa+1}\;\cdot\; c_S\;\cdot\; c_{\mathbb{B},N}
=\; \frac{C}{\kappa+1},
\end{aligned}
\label{eq:pkappa-mult}
\end{equation}
where, $C=c_V\cdot c_S\cdot c_{\mathbb{B},N}\in(0,1)$.
Consequently,
\[
\mathbb{E}[Y_{\kappa}]
=\sum_{i=1}^{\kappa} p_i
\;\ge\; \sum_{i=1}^{\kappa}\frac{C}{i+1}
\ \xrightarrow[\kappa\to\infty]{}\ \infty,
\]
which completes \textbf{(a)}.

\paragraph{\textbf{Proof of (b).}}
Starting from \eqref{eq:VarYkappa-expanded},
\begin{equation*}
\operatorname{Var}(Y_{\kappa})
= \sum_{i=1}^{\kappa} p_i(1-p_i)
  + 2 \sum_{1\le i<j\le \kappa} \Big(\mathbb{E}[X_i X_j] - p_i p_j\Big).
\end{equation*}
For the first term, \(p_i(1-p_i)\le p_i\), hence
\begin{equation*}
\sum_{i=1}^{\kappa} p_i(1-p_i)\ \le\ \sum_{i=1}^{\kappa} p_i\ =\ \mathbb{E}[Y_{\kappa}].
\end{equation*}
For the covariance term, since \(X_i X_j \le X_i\),
\begin{equation*}
\mathbb{E}[X_i X_j] - p_i p_j\ \le \mathbb{E}[X_i] - p_i p_j\ \le\ p_i(1-p_j)\ \le\ p_i,
\end{equation*}
and therefore
\begin{equation*}
2 \sum_{1\le i<j\le \kappa} \Big(\mathbb{E}[X_i X_j] - p_i p_j\Big)
\ \le\ 2 \sum_{i=1}^{\kappa} p_i
\ =\ 2\,\mathbb{E}[Y_{\kappa}].
\end{equation*}
Combining the two bounds gives
\begin{equation}
\operatorname{Var}(Y_{\kappa})
\ \le\ \mathbb{E}[Y_{\kappa}] + 2\,\mathbb{E}[Y_{\kappa}]
\ =\ 3\,\mathbb{E}[Y_{\kappa}],
\label{eq:Var-le-3EY}
\end{equation}
which proves \textbf{(b)}.

\end{proof}

\section{Experimental Validation} \label{sec:Sim}

\textcolor{black}{In this section,  we evaluate \textcolor{black}{\textbf{BRiDGE}}, our sampling–based planner, in setups of varying complexity. Our evaluation focuses on planning runtime, the horizon $H$ of the synthesized plan, effect of various biasing parameters and the 
\textcolor{black}{robustness of the algorithm under both \emph{exact} and \emph{over-approximate} implementations of the \texttt{Candidate} function, which in turn leads to \emph{initially feasible} or \emph{initially infeasible} symbolic plans~$\Pi$.}
Section~\ref{sec:PRM} summarizes the implementation of all utility functions—including \texttt{NavigReach}, \texttt{ManipReach}, \texttt{Gap}, and \texttt{Candidate}—as well as the baseline strategies used for comparison. Section~\ref{sec:scalability} presents scalability results across environments with increasing numbers of planes, blocks, and required plan horizons, highlighting the effect of symbolic biasing and the impact of \textcolor{black}{incorrect} candidates on runtime. Section~\ref{subsec:cs1} provides comparative experiments against uniform sampling and LLM-driven non-uniform sampling strategies. In Section~\ref{subsec:cs2}, we discuss the effect of the various biasing parameters on the runtimes of the sampling based planner. Finally, Section~\ref{subsec:cs3} reports hardware experiments on a heterogeneous two-robot system comprising a humanoid and a quadruped robot, demonstrating the practicality of our approach in real multi-level environments. All simulations were implemented in Python~3 and executed on a machine with an Intel Core i7-8565U CPU (1.8~GHz) and 16~GB RAM. Videos for the experiments can be found in \url{https://vimeo.com/1151395396?share=copy&fl=sv&fe=ci} and the code can be found in \url{https://github.com/kantaroslab/BRiDGE}
}


\subsection{\textcolor{black}{Implementation Details, Evaluation Metrics, \& Baselines}}\label{sec:PRM}

In what follows, we present implementation details used by  \textcolor{black}{\textbf{BRiDGE}}, including the utility functions introduced in Sections \ref{sec:PF} and \ref{sec:settingUP} (used to set up our planner), evaluation metrics, and baseline methods.

\textbf{Robot Navigation Capabilities:}\label{subsec:prm} \textcolor{black}{Throughout this section, we consider a holonomic robot, i.e., one that can follow any desired path. We assume the robot has a climbing policy that allows it to move from a plane $\mathcal{P}_i$ to a plane $\mathcal{P}_j$ \textcolor{black}{if their height difference is less than $\delta_H = 1.2$ units and their lateral separation is less than $\delta_L = 2$ units.} 
Under these specifications, we can efficiently determine whether a dynamically feasible path exists between two positions $\mathbf{x}_i$ and $\mathbf{x}_j$, which enables the construction of the utility function $\texttt{NavigReach}(\mathbf{x}_i, \mathbf{x}_j, \Omega(t))$.}
To speed up these reachability checks, we precompute a probabilistic roadmap (PRM) \cite{kavraki1996probabilistic}. The PRM takes the form of a graph $\mathcal{G}_G = (\mathcal{V}_G, \mathcal{E}_G)$ where the set of nodes contains obstacle-free robot configurations, and the set of edges edges represent dynamically feasible connections between nodes. Nodes may lie on both static planes ($\mathcal{P}^{\text{nm}}$) and on the surfaces of movable blocks ($\mathcal{P}^{\text{m}}$). Then, $\texttt{NavigReach}(\mathbf{x}_i, \mathbf{x}_j, \Omega^v)$ simply returns \texttt{True} if a path exists between $\mathbf{x}_i$ and $\mathbf{x}_j$ in $\mathcal{G}_G$ defined over an environment $\Omega^v$.
This PRM is initially constructed at the root node of the planning tree but evolves as the tree grows. Each newly generated tree node corresponds to a different arrangement of movable blocks, requiring local updates to the PRM:
(i)  \textcolor{black}{picking up a block removes all PRM nodes whose configurations lie on that
block’s top surface $\mathcal{P}_i$, and consequently all edges incident to those
nodes.} 
(ii) it may also add edges between nearby nodes (within the PRM connection radius) that become collision-free once the block is removed;
(iii) placing a block removes edges whose straight-line segments intersect the block’s new volume; and
(iv) \textcolor{black}{adds edges between nodes on the block’s top surface and nearby nodes on planes that become reachable through that placement.}
These incremental modifications ensure the PRM remains consistent with the evolving environment geometry and connectivity. We emphasize that more complex robot dynamics can be considered as long as a $\texttt{NavigReach}(\mathbf{x}_i, \mathbf{x}_j, \Omega(t))$ function is provided; for instance this function can be implemented using kino-dynamic planners \cite{csucan2009kinodynamic}.

\textbf{Robot Manipulation Capabilities:} 
We also assume that the robot has a gripper allowing it to manipulate movable blocks.
\textcolor{black}{We approximate $\texttt{ManipReach}(b_i,\Omega^v)$ by all collision-free robot 
configurations within a radius $\delta_M=1.9$ units of $b_i$, which suffices for both 
picking and placing operations in our implementation.}

\textbf{Implementation of Function $\texttt{Gap}(\ccalP_i,\ccalP_j,\Omega^v)$:}  To evaluate $\texttt{Gap}(\ccalP_i,\ccalP_j,\Omega^v)$, we simply check whether the PRM associated with node $v$ contains any path from a node on $\ccalP_i$ to a node on $\ccalP_j$; if none exists, $\texttt{Gap}$ returns \texttt{True}. \textcolor{black}{Note that this is a conservative implementation of the function defined in Definition~\ref{def:Gap}, as it considers only PRM nodes/configurations lying on $\ccalP_i$ and $\ccalP_j$, rather than all possible configurations on these planes.}

\textbf{Implementation of Function $\texttt{Candidate}(\ccalP_i,\ccalP_j)$:} 
\texttt{Candidate}: We implement this function with a lightweight geometric height check that returns a conservative set (\textcolor{black}{\emph{over-approximate} set that includes all true candidates and may add false positives}). The goal is to reduce the search space by pruning obviously incorrect choices. Concretely, we include $(b_n,\ccalP_k)$ in the output set of $\texttt{Candidate}(\ccalP_i,\ccalP_j)$ if placing $b_n$ \textcolor{black}{anywhere} on $\ccalP_k$ would make the top surface of $b_n$ within the robot’s step-height range \textcolor{black}{(determined by $\delta_H$ and $\delta_L$)} of both $\ccalP_i$ and $\ccalP_j$ (so the robot can step up or down onto it from either plane). Incorrect candidates are tolerated and later filtered by feasibility checks in Alg.~\ref{alg:sampling_planner}. \textcolor{black}{Note that, due to the right-polygonal-prism structure of all blocks, a pair $(b_n, \ccalP_k)$ returned by $\texttt{Candidate}(\ccalP_i,\ccalP_j)$, is an incorrect candidate only if there is insufficient space on $\ccalP_k$ to place $b_n$—either because the bottom surface of $b_n$ is larger than the top surface of $\ccalP_k$, or because multiple blocks already occupy $\ccalP_k$, leaving no room for $b_n$); see the placement requirements for movable blocks on non-movable planes in Section \ref{sec:PF}.  }

\textbf{Configuring the Proposed Sampling Strategy:} Unless otherwise specified, we configure our non-uniform sampling-strategy by setting (i) $p_{\text{plan}} = 0.9$ in $f_{\ccalV}$; (ii) $(p_1,p_2,p_3)=(0.85,\,0.15,\,0.05)$ in $f_{\ccalS}$; and (iii) $p_{\text{gap}}=0.9$ in $f_{\ccalC^{\,v}_{\text{block}}}$. Also, we use symbolic plan $\Pi$ generated by a BFS planner. 

\textbf{Baselines:} We compare \textcolor{black}{\textbf{BRiDGE}} with two baselines that differ from our planner only in their sampling strategy; both our method and the baselines are run for $K_{\text{max}}=10,000$ iterations and $N=100$. We consider the following setups:\footnote{\textcolor{black}{We also attempted to directly solve the considered problem using LLMs (GPT 5.1) as end-to-end planners (no sampling was involved). While LLMs are capable of generating high-level plans, defined as sequences of triplets (see the non-uniform LLM-driven sampling strategy), we observed that they cannot perform low-level spatial reasoning and, therefore, cannot generate low-level plans with feasible robot poses and object drop-off locations.} 
}

(i) \textit{Uniform Sampling}: Node selection in Alg.~\ref{alg:sampling_planner} is uniform over the search tree, and triplets (block, placement plane, intent) are drawn uniformly from the candidate superset $\ccalS$. We still use the biased strategy ($f_{\ccalC^{v}_{\text{block}}}$) for sampling block configurations; \textcolor{black}{unless otherwise specified, we set $p_{\text{gap}}=0.9$ to ensure meaningful comparisons against our proposed non-uniform sampling strategy. This choice is deliberate, as we observed that a fully uniform strategy cannot solve planning problems with horizon $H\geq 3$ due to the large search space. Notice that this setup does not require using the plan $\Pi$ at all.
We emphasize that this framework is probabilistically complete as it can be recovered by our method by appropriately setting $p_{\text{plan}}, p_1, p_2$, and $p_3$ as discussed in Section \ref{sec:biased-sampling}, without violating Theorem \ref{thm:pc}.}

\textcolor{black}{(ii) \textit{Non-Uniform LLM-driven Sampling}: This framework shares exactly the same sampling strategy as ours (including values for all parameters). The only difference is that we replace the BFS planner with an LLM, as discussed in Remark~\ref{rem:LLM}, \textcolor{black}{and thus this framework is probabilistically complete as well.}
Concretely, we prompt GPT~5.1 to generate a symbolic plan \(\Pi\) over planes and blocks using the same high-level information that is available to the BFS planner: the set of planes and blocks, gap relations between planes, the available high-level actions (\textsc{GoTo}, \textsc{Pick}, \textsc{Place} with placement intent), and simple task constraints.
This information is provided in a structured, JSON-like prompt that describes the environment and admissible actions, and the LLM returns a sequence of symbolic actions in exactly the same triplet format as the BFS-produced plan \(\Pi\); \textcolor{black}{example for our prompts can be found in} \url{https://github.com/kantaroslab/BRiDGE}. Since the input--output interface matches that of the BFS planner, the LLM-guided plan can be used as a drop-in replacement within our sampling-based planner.
}

\textbf{Evaluation Metrics:} As evaluation metrics, we use (i)
the runtime to compute the first feasible plan, (ii) the horizon $H$ of the designed plan, \textcolor{black}{(iii) the size of the constructed trees}, 
and (iv) number of times  a symbolic plan $\Pi$ was regenerated \textcolor{black}{(also referred to as queries to the BFS planner)} due to biasing towards infeasible directions. The reported metrics are averaged over 3 trials.

\subsection{Scalability Analysis with Problem Complexity}\label{sec:scalability}

\textcolor{black}{In this section, we evaluate the scalability of \textcolor{black}{\textbf{BRiDGE}} with respect to the number of non-movable planes, the number of movable blocks, and the planning horizon $H$ (i.e., the number of actions required to make the goal region accessible). We consider cases where the minimum possible horizon, denoted by $H_{\text{min}}$, is $H_{\text{min}} \in \{2,6\}$. For each horizon, we generate nine environments by taking all permutations of three plane counts and three block counts. Specifically, when $H_{\text{min}} = 2$, we consider numbers of non-movable planes $|\mathcal{P}^{\text{nm}}| \in \{3,6,9\}$ and numbers of movable blocks $|\mathcal{P}^{m}| \in {2,5,10}$. Similarly, when $H = 6$, we consider $|\mathcal{P}^{\text{nm}}| \in \{6,9,12\}$ and $|\mathcal{P}^{m}| \in \{3,6,10\}$.}

\begin{table}[t]
\centering
\caption{Results for Horizon $H_{\text{min}}=2$ using the BFS-based biasing strategy. 
Reported for BFS bias under: (\emph{initially feasible plan $\Pi$}) and (\emph{initially infeasible plan $\Pi$}: time, queries).} 
\label{tab:horizon2}
\begin{tabular}{|c|c|c|c|}
\hline
\textbf{$|\ccalP^{\text{nm}}|$} & \textbf{$|\ccalB^{\text{m}}|$} 
& \textbf{BFS bias} 
& \textbf{BFS bias} \\[-2pt]
 &  & \small{[initial $\Pi$ feasible]} & \small{[initial $\Pi$ infeasible]} \\
 &  &  \small{(time (sec))} & \small{(time (sec), queries)}\\
\hline
3  & 2  & 7.3 & 18, 2 queries \\
3  & 5  & 7.9 & 26, 3 queries \\
3  & 10 & 7.1 & 52, 5 queries \\
6  & 2  & 7.2 & 20, 2 queries \\
6  & 5  & 7.3 & 33, 4 queries \\
6  & 10 & 7.2 & 56, 5 queries \\
9  & 2  & 7.1 & 19, 2 queries \\
9  & 5  & 8.2 & 28, 3 queries \\
9  & 10 & 7.9 & 57, 5 queries \\
\hline
\end{tabular}
\end{table}

\begin{table}[t]
\centering
\caption{Results for Horizon $H_{\text{min}}=6$ using the BFS-based biasing strategy. 
Reported for BFS bias under: (\emph{initially feasible plan $\Pi$}) and (\emph{initially infeasible plan $\Pi$}: time, queries).}
\label{tab:horizon6}
\begin{tabular}{|c|c|c|c|}
\hline
\textbf{$|\ccalP^{\text{nm}}|$} & \textbf{$|\ccalB^{\text{m}}|$} 
& \textbf{BFS bias} 
& \textbf{BFS bias} \\[-2pt]
 &  & \small{[initial $\Pi$ feasible]} & \small{[initial $\Pi$ infeasible]} \\
 &  &  \small{(time (sec))} & \small{(time (sec), queries)}\\
\hline
6  & 3  & 7.6 & 24, 3 queries \\
6  & 6  & 7.5 & 35, 4 queries \\
6  & 10 & 7.9 & 49, 5 queries \\
9  & 3  & 8.1 & 26, 3 queries \\
9  & 6  & 7.7 & 38, 4 queries \\
9  & 10 & 7.9 & 54, 6 queries \\
12 & 3  & 7.8 & 26, 3 queries \\
12 & 6  & 8.3 & 39, 4 queries \\
12 & 10 & 8.1 & 48, 5 queries \\
\hline
\end{tabular}
\end{table}


\textcolor{black}{
In all case studies, 
\textcolor{black}{we observed that when the \texttt{Candidate} function was \emph{exact} (i.e., produced no false positives), the initial high-level plans~$\Pi$ generated by the BFS planner were \emph{initially feasible}.}
\footnote{\textcolor{black}{In a general, a plan $\Pi$ may still be infeasible even if the \texttt{Candidate} function is exact, i.e., it returns only valid blocks; see Section \ref{subsec:fail-refresh}.}} To evaluate robustness against imperfect implementations of the \texttt{Candidate} function and potentially infeasible plans, we tested two settings:
(i) the initial plan $\Pi$ is feasible; (ii) the initial plan $\Pi$ is infeasible, enforced by manually inserting incorrect candidate blocks in the output of the \texttt{Candidate} function (\textcolor{black}{thus mimicking an over-approximate implementation of the function}). The results for $H_{\text{min}}= 2$ are summarized in Table~\ref{tab:horizon2} under both settings. When the initial plan $\Pi$ is feasible, the runtimes of our method to compute the first feasible solution $\tau$, remain almost constant (approximately 7 s) across all environments, regardless of their complexity. This occurs because the BFS plan $\Pi$ consistently steers the sampling strategy toward feasible actions, while still allowing some random exploration for geometric checks and trajectories not prescribed by $\Pi$. When $\Pi$ may include infeasible steps, these are pruned as failures are detected, triggering recomputation of $\Pi$, as described in Section~\ref{subsec:fail-refresh}.} The “queries’’ in the third column of the table quantify this overhead, i.e., the number of times $\Pi$ was recomputed. Runtimes scale roughly with the number of re-plans (e.g., 18 s for 2 queries up to 57 s for 5 queries). 
\textcolor{black}{Note that in all the cases the horizon of the designed plan was $H=2$ meaning the planner found the plan with the shortest possible horizon.}
Similar trends appear for the case $H_{\text{min}}= 6$, summarized in Table~\ref{tab:horizon6}; \textcolor{black}{the horizon of the designed plan was $H=H_{\text{min}}=6$ in all cases.} The main difference compared to the $H_{\text{min}} = 2$ case is that runtimes are slightly higher, which is expected as the planning horizon increases. 



\begin{figure*}[t]
    \centering
    \subfigure[]{\label{fig:case2a}\includegraphics[width=0.23\textwidth]{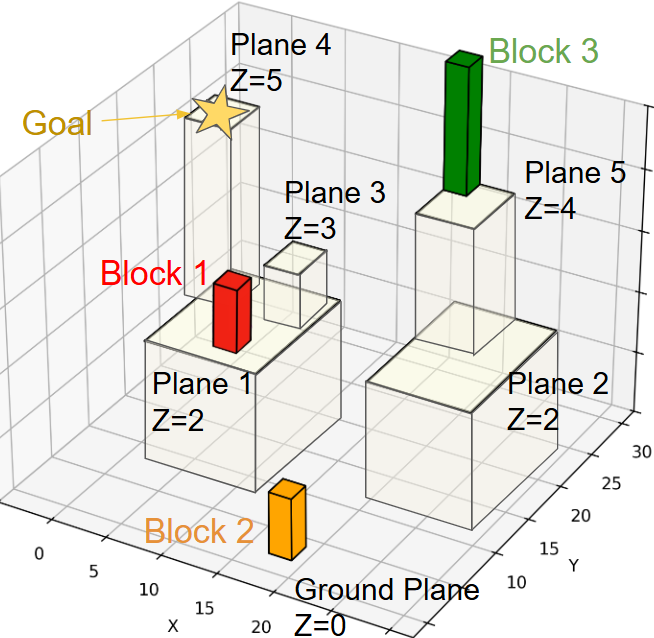}} \hfill
    \subfigure[]{\label{fig:case2b}\includegraphics[width=0.23\textwidth]{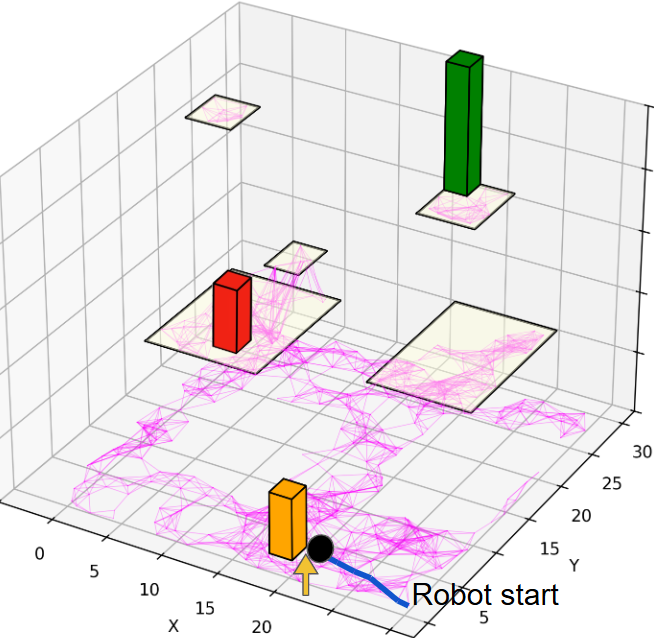}} \hfill
    \subfigure[]{\label{fig:case2c}\includegraphics[width=0.23\textwidth]{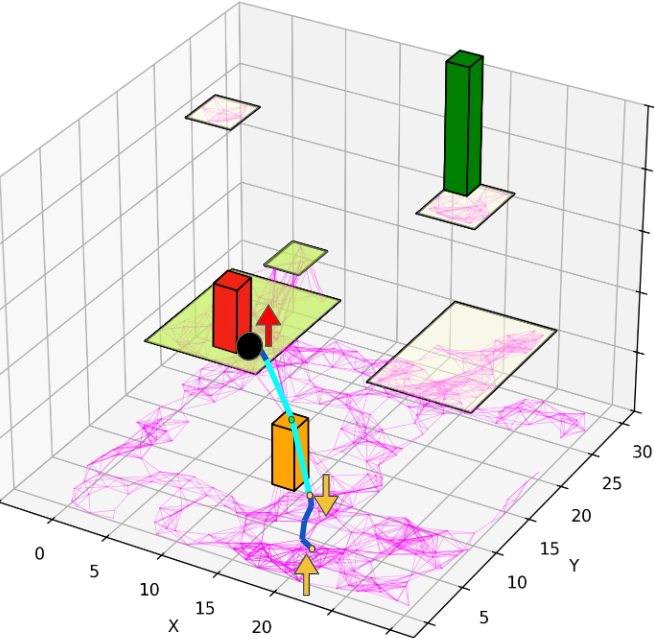}} \hfill
    \subfigure[]{\label{fig:case2d}\includegraphics[width=0.23\textwidth]{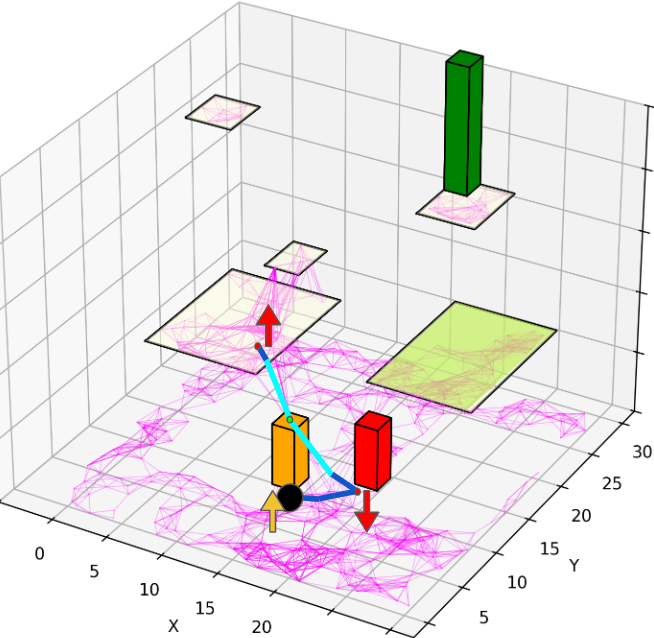}} \\
    \subfigure[]{\label{fig:case2e}\includegraphics[width=0.23\textwidth]{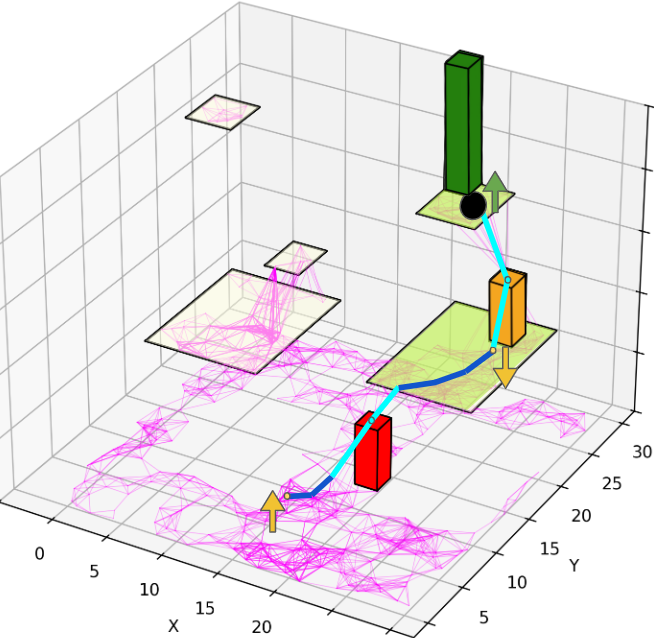}} \hfill
    \subfigure[]{\label{fig:case2f}\includegraphics[width=0.23\textwidth]{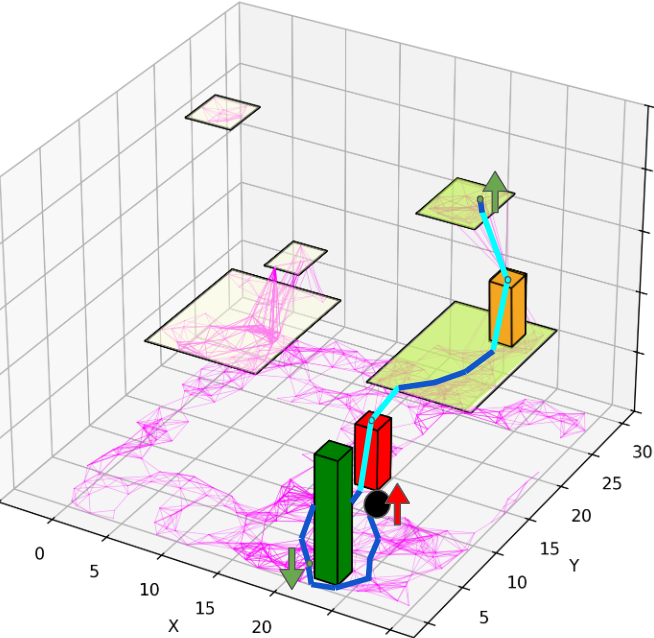}} \hfill
    \subfigure[]{\label{fig:case2g}\includegraphics[width=0.23\textwidth]{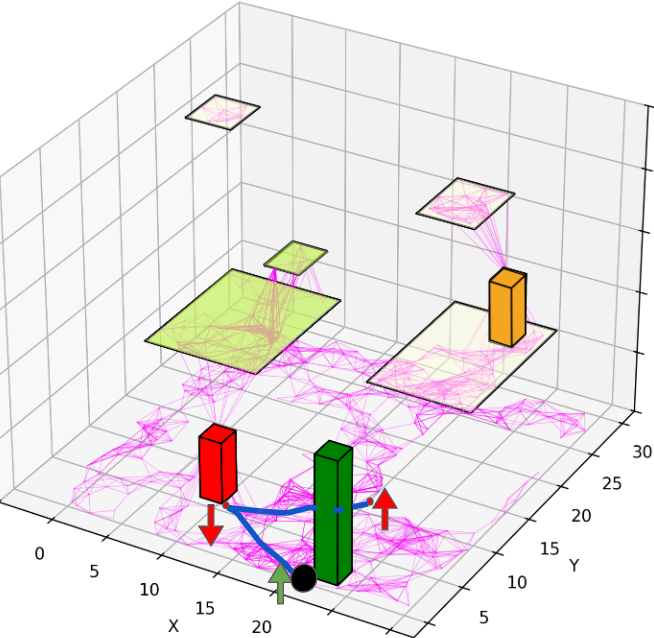}} \hfill
    \subfigure[]{\label{fig:case2h}\includegraphics[width=0.23\textwidth]{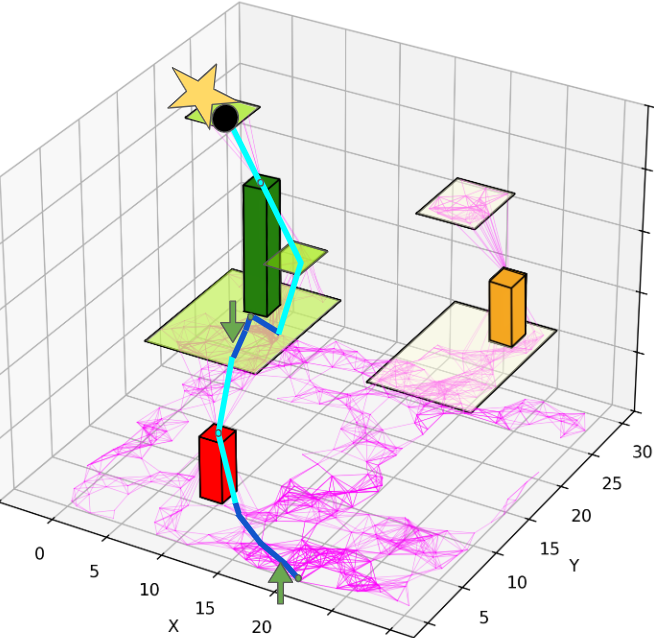}}
    \caption{Snapshots after each $\texttt{Move}(\bbb_i, \bbb_i')$ in the recovered plan $\tau$.  Robot trajectory on a plane is shown in \textcolor{blue}{blue}; inter–plane trajectories are in \textcolor{cyan}{cyan}. Upward arrows mark \emph{block pick} action; downward arrows mark \emph{block place} action. \textcolor{magenta}{Pink} lines show the PRM connectivity (Sec. \hyperref[subsec:prm]{6.1})). Planes reachable from robot position are highlighted in light green. 
    \textbf{Task.} Fig.~\ref{fig:case2a} shows the initial environment (Ground plus Planes~1–5, and heights (Z)). Blocks~1 (red) and~2 (orange) are 1\,unit tall; Block~3 (green) is 2\,units; the robot can ascend/descend at most $\delta_H=1.2$\,unit. The start is on Ground; the goal ($\star$) is on Plane~4. 
    Fig.~\ref{fig:case2b} shows the initial PRM. 
    Figs.~\ref{fig:case2c}–\ref{fig:case2e}: Blocks~1 and~2 are used as intermediate steps to reach Plane~5 and \emph{pick} Block~3. 
    Fig.~\ref{fig:case2f}: Block~3 is \emph{placed} temporarily on Ground. 
    Fig.~\ref{fig:case2g}: Block~1 is reused to establish connectivity with Plane~1. 
    Fig.~\ref{fig:case2h}: Block~3 is \emph{placed} to connect Planes~3 and~4, enabling the robot to reach the goal region.
}
    \label{fig:case2}
\end{figure*}

\textcolor{black}{To more clearly highlight the complexity of the planning problems addressed by our method, and the long-horizon reasoning and planning behaviors it exhibits, we next provide a detailed walkthrough of the plan generated by our algorithm for a representative case study with parameters $H_{\text{min}} = 6$, $|\mathcal{P}^{\text{nm}}| = 6$ (including the ground plane), and $|\mathcal{B}^{\text{m}}| = 3$; see Fig.~\ref{fig:case2}. The heights of the non-movable (non-ground) planes $\ccalP_1$, $\ccalP_2$, $\ccalP_3$, $\ccalP_4$, and $\ccalP_5$ are $2, 2, 3, 5,$ and $4$. The goal region lies on plane $\ccalP_4$. The robot is initially located on the ground plane and therefore the goal region is not accessible given its navigation capabilities discussed in Section~\ref{sec:PRM}. 
To reach the goal region, the robot leverages three movable blocks: block $b_1$ of height $h_1 = 1$ initially located on $\mathcal{P}_1$; block $b_2$ of height $h_2 = 1$ initially located on the ground plane; and block $b_3$ of height $h_3 = 2$ initially located on $\mathcal{P}_5$. Observe in Fig.~\ref{fig:case2} that block $b_3$ is the only block that can be used to bridge the gap between $\mathcal{P}_3$ and $\mathcal{P}_4$, where the goal region lies, by placing it on $\ccalP_1$. However, $b_3$ is not initially accessible to the robot. Access requires first reaching $\mathcal{P}_2$ and then $\mathcal{P}_5$, but neither plane is directly reachable from the ground plane, and $\mathcal{P}_5$ is not reachable from $\mathcal{P}_2$ either.
Given the robot's navigation and manipulation capabilities, the only way to access $\mathcal{P}_5$ is to use both blocks $b_1$ and $b_2$: one to bridge the gap between the ground plane and $\mathcal{P}_2$, and the other to bridge the gap between $\mathcal{P}_2$ and $\mathcal{P}_5$. However, only $b_2$ is initially accessible (Fig.~\ref{fig:case2b}). Thus, the robot first moves $b_2$ next to $\mathcal{P}_1$, allowing it to climb to $\mathcal{P}_1$ (Fig.~\ref{fig:case2c}), and then relocates $b_1$ to the ground plane to bridge the ground and $\ccalP_2$ (Fig.~\ref{fig:case2d}). Note that bringing $b_3$ to the ground requires removing the previous bridge (block $b_2$) between the ground plane and $\ccalP_1$ and then using $b_2$ to bridge planes $\ccalP_2$ and $\ccalP_5$. 
This enables the robot to reach $\mathcal{P}_5$ (Fig.~\ref{fig:case2e}) and move $b_3$ down to the ground plane (Fig.~\ref{fig:case2f}). At this point, the robot and blocks $b_1$ and $b_3$ are all on the ground. The robot temporarily places $b_3$ on the ground to re-bridge the gap between the ground plane and $\ccalP_1$ using $b_1$. Specifically, the robot pushes $b_1$ back next to $\mathcal{P}_1$ (Fig.~\ref{fig:case2g}), grabs $b_3$, climbs on $b_1$ and then to $\ccalP_1$ while carrying $b_3$, and finally places $b_3$ on $\mathcal{P}_1$, thereby bridging the gap between $\mathcal{P}_3$ and $\mathcal{P}_4$ (Fig.~\ref{fig:case2h}).
Finally, the robot climbs from $\mathcal{P}_2$ to $\mathcal{P}_3$, then onto $b_3$, and ultimately reaches $\mathcal{P}_4$ to access the goal region (Fig.~\ref{fig:case2h}).
%
}

%
%
Observe in this example that our method computes \emph{indirect} and non-trivial strategies such as: temporarily bridging an \emph{alternate} gap to access a critical block (e.g., creating intermediate bridges to reach $\ccalP_5$);
performing a \emph{temporary placement} to free the robot for another action and subsequently retrieving and redeploying that block to complete the original objective
(e.g., placing $b_3$ briefly on the ground and then using it to bridge $\ccalP_3$ and $\ccalP_4$); and \emph{reusing} the same block
to serve different purposes at different times (e.g., $b_2$ first bridges
Ground–$\ccalP_1$ and later $\ccalP_2$–$\ccalP_5$).

\subsection{Comparative Experiments against Baselines}\label{subsec:cs1}

\textcolor{black}{In this section, we evaluate the performance of our algorithm when paired with uniform and non-uniform LLM-driven sampling strategies discussed in Section~\ref{sec:PRM} on the case studies introduced in Section~\ref{sec:scalability}. To simplify the analysis, 
\textcolor{black}{we restrict the evaluation to the \emph{exact} \texttt{Candidate} setting (i.e., no false positives)};
for example, combining potential LLM errors with incorrect candidates from \texttt{Candidate} would make it difficult to isolate and evaluate the specific impact of the LLM.} 

\begin{table}[t]
\centering
\caption{Comparison for Horizon $=2$. Reported planning times for: BFS bias, Uniform sampling, and LLM bias.}
\label{tab:horizon2Comparison}
\begin{tabular}{|c|c|c|c|c|}
\hline
\textbf{$|\ccalP^{\text{nm}}|$} & \textbf{$|\ccalB^{\text{m}}|$} 
& \textbf{BFS bias} 
& \textbf{Uniform} 
& \textbf{LLM bias} \\
 &  &  \small{(sec)} &  \small{(sec)} &  \small{(sec)}\\
\hline
3  & 2  & 7.3 & 8   & 16 \\
3  & 5  & 7.9 & 12  & 20 \\
3  & 10 & 7.1 & 18  & 28 \\
6  & 2  & 7.2 & 26  & 21 \\
6  & 5  & 7.3 & 50  & 25 \\
6  & 10 & 7.2 & 80  & 31 \\
9  & 2  & 7.1 & 46  & 25 \\
9  & 5  & 8.2 & 75  & 32 \\
9  & 10 & 7.9 & 110 & 38 \\
\hline
\end{tabular}
\end{table}

\begin{table}[t]
\centering
\caption{Comparison for Horizon $=6$. Reported planning times for: BFS bias, Uniform sampling, and LLM bias.}
\label{tab:horizon6Comparison}
\begin{tabular}{|c|c|c|c|c|}
\hline
\textbf{$|\ccalP^{\text{nm}}|$} & \textbf{$|\ccalB^{\text{m}}|$} 
& \textbf{BFS bias} 
& \textbf{Uniform} 
& \textbf{LLM bias} \\
 &  &  \small{(sec)} &  \small{(sec)} &  \small{(sec)}\\
\hline
6  & 3  & 7.6 & N/A & 107 \\
6  & 6  & 7.5 & N/A & 121 \\
6  & 10 & 7.9 & N/A & 205 \\
9  & 3  & 8.1 & N/A & 144 \\
9  & 6  & 7.7 & N/A & 156 \\
9  & 10 & 7.9 & N/A & 272$^{*}$ \\
12 & 3  & 7.8 & N/A & 160 \\
12 & 6  & 8.3 & N/A & 263$^{*}$ \\
12 & 10 & 8.1 & N/A & 209 \\
\hline
\end{tabular}
\end{table}

\textcolor{black}{The comparative results for $H_{\text{min}}=2$ and $H_{\text{min}}=6$ are shown in Tables~\ref{tab:horizon2Comparison} and \ref{tab:horizon6Comparison}, respectively. Across both tables, a consistent picture emerges. For $H_{\text{min}}=2$ (Table~\ref{tab:horizon2Comparison}), the runtime for \emph{uniform-sampling} grows rapidly with environment size (e.g., from $8$ secs in environments with $3$ planes and $2$ movable blocks to $110$ secs in environments with $9$ planes and $10$ movable blocks) even though the horizon $H_{\text{min}}$ is fixed. This occurs because the space that needs to be explored and, specifically, the number of triplets $s$, grows exponentially as the number of blocks and planes increases. Thus, with no guidance on which triplets to prioritize, the runtimes significantly increase.
The \emph{LLM-biased sampling strategy} performs worse than the uniform-sampling framework when the number of non-movable planes and movable blocks is small. For instance, in environments with $3$ planes and $2$ movable blocks, the runtime of the LLM-biased strategy is $16$ secs, compared to $8$ secs required by uniform sampling. However, the LLM-biased strategy outperforms uniform sampling as the environment size increases: in environments with $9$ planes and $10$ movable blocks, its runtime is $38$ secs versus $110$ secs for uniform sampling.
Our method outperforms both baselines, achieving significantly lower runtimes ($\approx 7$ secs) that remain nearly constant across all case studies. Note that the LLM generated plan $\Pi$ was initially feasible in all nine cases providing `feasible' directions to the sampling strategy 
(just as the BFS planner did as discussed in Section \ref{sec:scalability}); however, it is slower than our method \textcolor{black}{(and the uniform-based strategy in small-scale problems)} due to the additional inference cost of querying GPT~5.1.} 
\textcolor{black}{We also note that the horizons of the returned plans were identical across all methods: every approach (BFS bias, uniform sampling, and LLM bias) successfully recovered the shortest-horizon plan with $H=H_{\min}=2$. 
The tree $\ccalT$ built using Alg. \ref{alg:sampling_planner} by both the BFS-biased and LLM-biased planners resulted in compact trees—between $3$ and $5$ nodes in all cases—reflecting their strong directional guidance and minimal need for exploration. 
In contrast, the uniform-sampling planner required substantially larger trees, with $|\ccalV_T|$ increasing from $4$ nodes in the simplest case to $57$ nodes in the largest environment, mirroring the growth in runtime and the exponential expansion of the search space.}

As for the case studies with $H_{\text{min}}=6$ (Table~\ref{tab:horizon6Comparison}), we observed that \emph{uniform sampling} fails to find a solution within the iteration budget ($K_\text{max}$) on all nine instances. \textcolor{black}{Moreover, the runtimes for the LLM-based sampling strategy are significantly larger than the corresponding BFS-biased runs$-$ranging from $107$ secs to $272$ secs$-$primarily because of the inherent latency of the LLM in generating each high-level plan $\Pi$; this cost dominates even when the symbolic plan it produces is correct. In some cases—marked with an asterisk in Table~\ref{tab:horizon6Comparison}—the LLM additionally produced an incorrect symbolic action within $\Pi$, even though the \texttt{Candidate} function in this comparison contains only correct tuples. Such errors stem from the natural fallibility of LLMs when reasoning over long-horizon symbolic tasks. When an incorrect action is proposed, the geometry checks in Alg.~\ref{alg:sample_move} reject it, and we re-query the LLM to obtain a revised plan, analogous to the failure-refresh mechanism in Section~\ref{subsec:fail-refresh}. These re-planning events introduce extra overhead, explaining the noticeably higher runtimes in the starred entries. In contrast, as discussed in Section~\ref{sec:scalability}, our proposed method combined with the BFS planner achieves runtimes that remain consistently around $7.5$–$8.3$ secs across all evaluated problem sizes.}
\textcolor{black}{For all evaluated instances, the horizon of the plans computed \textcolor{black}{using BFS-driven and LLM-driven} was $H=H_{\min}=6$. 
Across problem sizes, both our BFS-biased planner and the LLM-biased planner produced similarly small trees, with $|\ccalV_T|$ ranging from $7$ to $13$ nodes reflecting the strong directional bias in both methods. }

\subsection{Effect of Bias Parameters \textcolor{black}{and Correctness of Plan $\Pi$}}\label{subsec:cs2}
\textcolor{black}{In this section, we evaluate the sensitivity of our algorithm to the parameters $p_1, p_2, p_3$, and $p_{\text{plan}}$ used in the proposed BFS-driven non-uniform sampling strategy. Specifically, we consider the following values:
$p_{\text{plan}} \in \{0.2, 0.5, 0.8\}$ (used in \eqref{eq:fv_define}) and
$(p_1, p_2, p_3) \in \{(0.2, 0.4, 0.4),\ (0.5, 0.25, 0.25),\ (0.9, 0.05, 0.05)\}$ (used in \eqref{eq:fs_define}).
For each parameter configuration, we report the average runtime required to compute the first feasible solution, along with the corresponding number of tree nodes generated by Alg.~\ref{alg:sampling_planner}, denoted by $|\mathcal{V}_T|$.}

\textcolor{black}{First, we consider the case study discussed in Section \ref{sec:scalability}, with $|\mathcal{P}^{\text{nm}}|=6$, $|\mathcal{B}^{\text{m}}|=3$, and $H_{\text{min}}=6$. To isolate the effect of biasing from potentially incorrect outputs of the function \texttt{Candidate}, we configure our planner using only \emph{correct} outputs of this function, i.e., blocks that can indeed bridge gaps. In this case study, using an \emph{exact} \texttt{Candidate} function, the BFS-generated plan~$\Pi$ was initially feasible in all trials.}\footnote{\textcolor{black}{Note that in a general a plan $\Pi$ may still infeasible even with `exact' \texttt{Candidate} function; see Section \ref{subsec:fail-refresh}.}} 
The results are shown in Table~\ref{tab:bias_sensitivity_h6_p6_b3}. We observe a consistent trend: increasing bias toward the BFS plan \(\Pi\), i.e., larger $p_{\text{plan}}$ and \(p_1\), \emph{reduces planning time} as the tree prioritizes expansions toward correct \(\Pi\)-recommended triplets/directions. In parallel, \(|\ccalV_T|\) \emph{drops sharply} as bias toward \(\Pi\) increases, aligning with the intuition that stronger bias toward correct high-level plans $\Pi$ may accelerate plan synthesis.

\begin{table}[t]
\centering
\caption{Bias sensitivity for $H{=}6$, 6 planes, 3 blocks (correct candidates only).} 
\label{tab:bias_sensitivity_h6_p6_b3}
\begin{tabular}{|c|c|c|c|}
\hline
\multirow{2}{*}{$(p_1,p_2,p_3)$} & \multicolumn{3}{c|}{$p_{\text{plan}}$; \textbf{time (s), ($|\ccalV_T|$ nodes)}} \\
\cline{2-4}
 & \textbf{0.2} & \textbf{0.5} & \textbf{0.8} \\
\hline
(0.2, 0.4, 0.4)   & 14.1, (34)  & 11.4, (19)  & 9.7, (13)  \\
(0.5, 0.25, 0.25) & 10.7, (17)  & 9.6, (12)   & 8.75, (10)  \\
(0.9, 0.05, 0.05) & 9.5, (15)   & 8.7, (9)    & 8.5, (7)  \\
\hline
\end{tabular}
\end{table}

\begin{table}[t]
\centering
\caption{Bias sensitivity for $H{=}2$, 3 planes, 2 blocks, with one infeasible candidate and a fixed high-level plan $\Pi$ (no re-planning).}
\label{tab:bias_sensitivity_h2_p3_b2_bad}
\begin{tabular}{|c|c|c|c|}
\hline
\multirow{2}{*}{$(p_1,p_2,p_3)$} & \multicolumn{3}{c|}{$p_{\text{plan}}$; \textbf{time (s)}, ($|\ccalV_T|$ nodes)} \\
\cline{2-4}
 & \textbf{0.2} & \textbf{0.5} & \textbf{0.8} \\
\hline
(0.2, 0.4, 0.4)   & 7.8, (4)  & 12.8, (4) & 22.1, (5) \\
(0.5, 0.25, 0.25) & 10.6, (4) & 15.3, (3) & 28.5, (4) \\
(0.9, 0.05, 0.05) & 16.6, (3) & 35.2, (4) & 57.0, (4) \\
\hline
\end{tabular}
\end{table}

\textcolor{black}{Second, to complement the above case—which isolates the effect of biasing when $\Pi$ is correct—we next examine a contrasting scenario where biasing interacts with an \emph{infeasible} high-level plan. }
Specifically, here, we consider a smaller instance with horizon $H_{\text{min}}{=}2$, three non-movable planes, and two movable blocks.  For one of the gaps between the non-movable planes, only a single block can physically create a bridge, but the function \texttt{Candidate} returns two options: one valid and one infeasible. This results in plan $\Pi$ that contains an infeasible triplet.
We deliberately disable the re-planning mechanism so that the BFS planner is invoked once, and the resulting plan $\Pi$ is held fixed even if the corresponding triplet never succeeds.  The results are summarized in Table~\ref{tab:bias_sensitivity_h2_p3_b2_bad}.  Unlike Table~\ref{tab:bias_sensitivity_h6_p6_b3}, here increasing the bias towards the (infeasible) plan $\Pi$ consistently \emph{increases} the planning time: higher $p_1$ and larger $p_{\text{plan}}$ make the planner spend most of its effort expanding the tree toward infeasible directions, leaving too little exploration to discover the valid candidate.  In the most extreme setting $(p_1,p_2,p_3){=}(0.9,0.05,0.05)$ and $p_{\text{plan}}{=}0.8$, the average planning time (without re-planning $\Pi$) reaches $57\,$s.  However, if we allow on-the-fly revision of $\Pi$, as discussed in Section \ref{subsec:fail-refresh}, the same parameter setting solves the problem in $12\,$s with a single re-planning of $\Pi$. \textcolor{black}{We note that \(|\ccalV_T|\) remains nearly unchanged across bias settings, since biased sampling toward an infeasible direction produces many rejected samples but few successful node insertions.} This experiment highlights that aggressive bias is highly beneficial when the high-level plan is correct, but can be detrimental under incorrect candidate unless coupled with the repair mechanism discussed in Section \ref{subsec:fail-refresh}.

\begin{rem}[Case-Study Selection]\label{rem:biasEffect}
    \textcolor{black}{For completeness, we note that the two case-study setups used above (horizon $H_{\text{min}}{=}6$ with \textcolor{black}{exact} \texttt{Candidate} function (no false positives), and horizon $H_{\text{min}}{=}2$ with \textcolor{black}{over-approximate} \texttt{Candidate} function) are chosen intentionally. 
    For the $H_{\text{min}}{=}6$ case, if an incorrect candidate 
    is introduced and re-planning is disabled, the planner fails to find a solution \textcolor{black}{within $K_{\text{max}}=10,000$ iterations}, under all bias configurations—making a bias-sensitivity table uninformative.  
    Conversely, for $H_{\text{min}}{=}2$, when all candidates are correct, the problem is so short-horizon that all bias settings yield comparable runtimes (on the order of $5$–$6$ seconds) and very small number of nodes, offering no meaningful trend to report. For these reasons, we present only the informative cases above.}
\end{rem}

\begin{figure}[t]
    \centering
    \includegraphics[width=\linewidth]{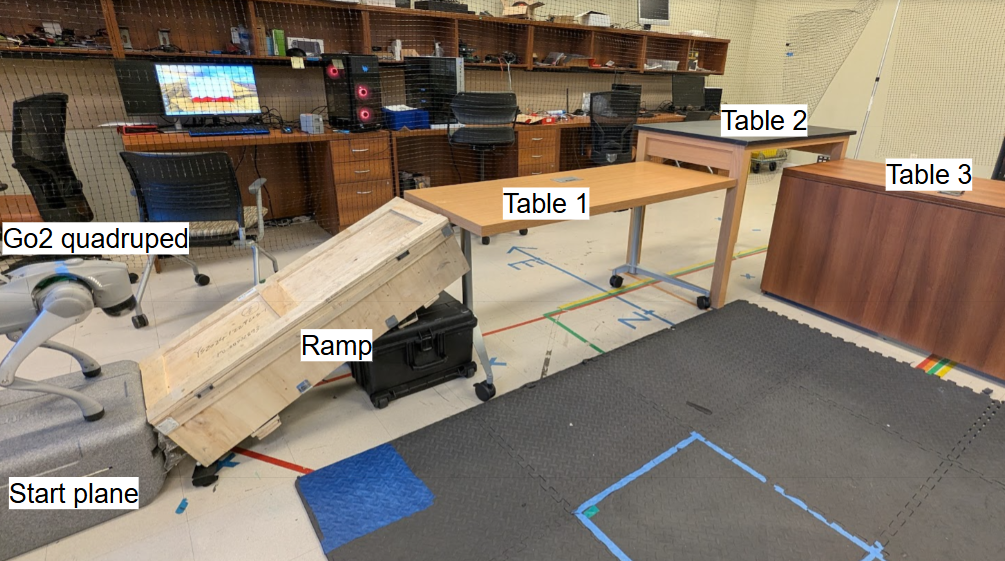}
    \caption{Hardware demonstration setup.
    The Go2 starts at the bottom of the ramp and needs to reach Table 3. 
    }
    \label{fig:hardware_setup}
\end{figure}

\begin{figure*}[t]
    \centering
    \subfigure[]{\label{fig:hw1}\includegraphics[width=0.3\textwidth]{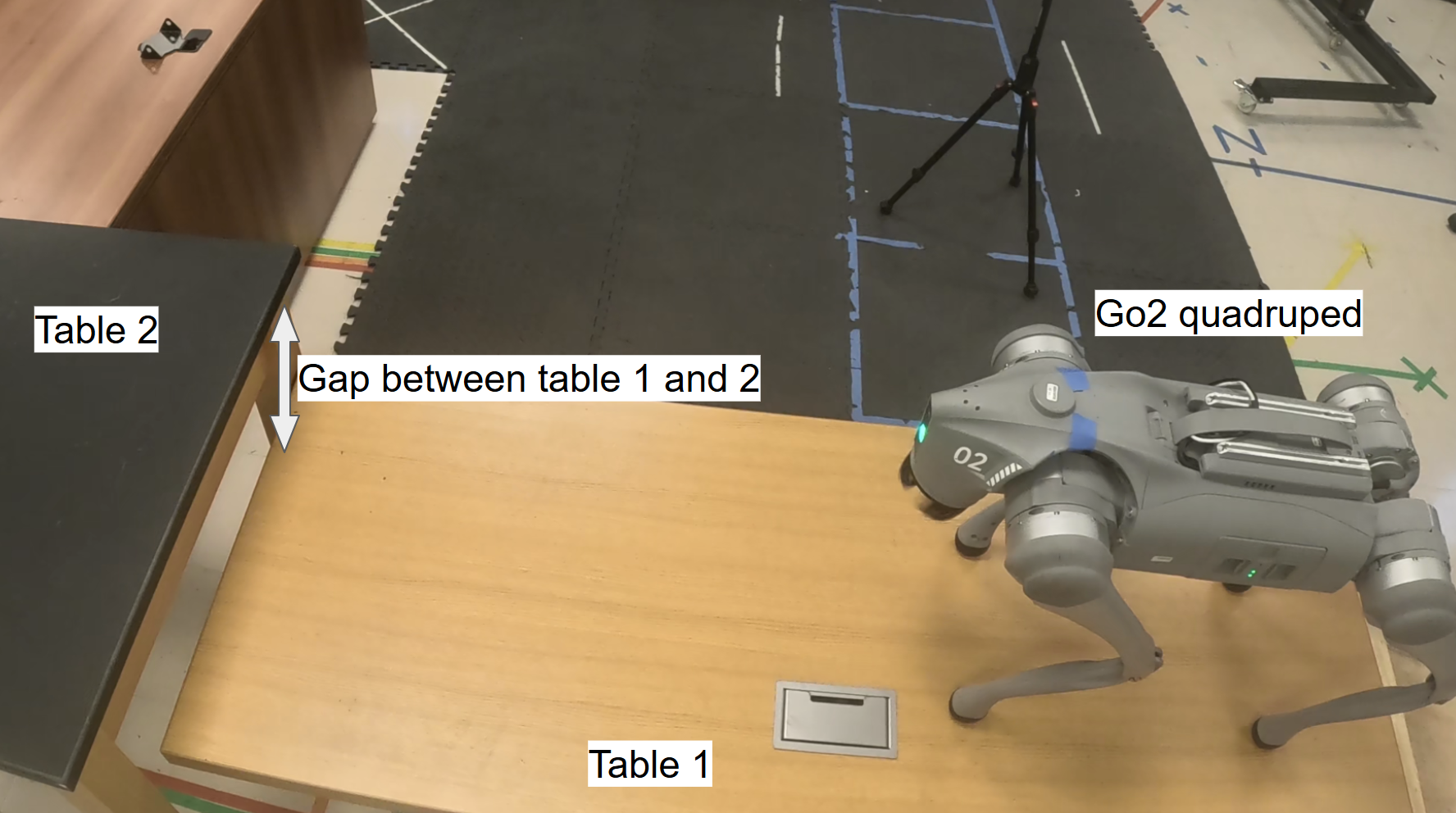}} \hfill
    \subfigure[]{\label{fig:hw2}\includegraphics[width=0.3\textwidth]{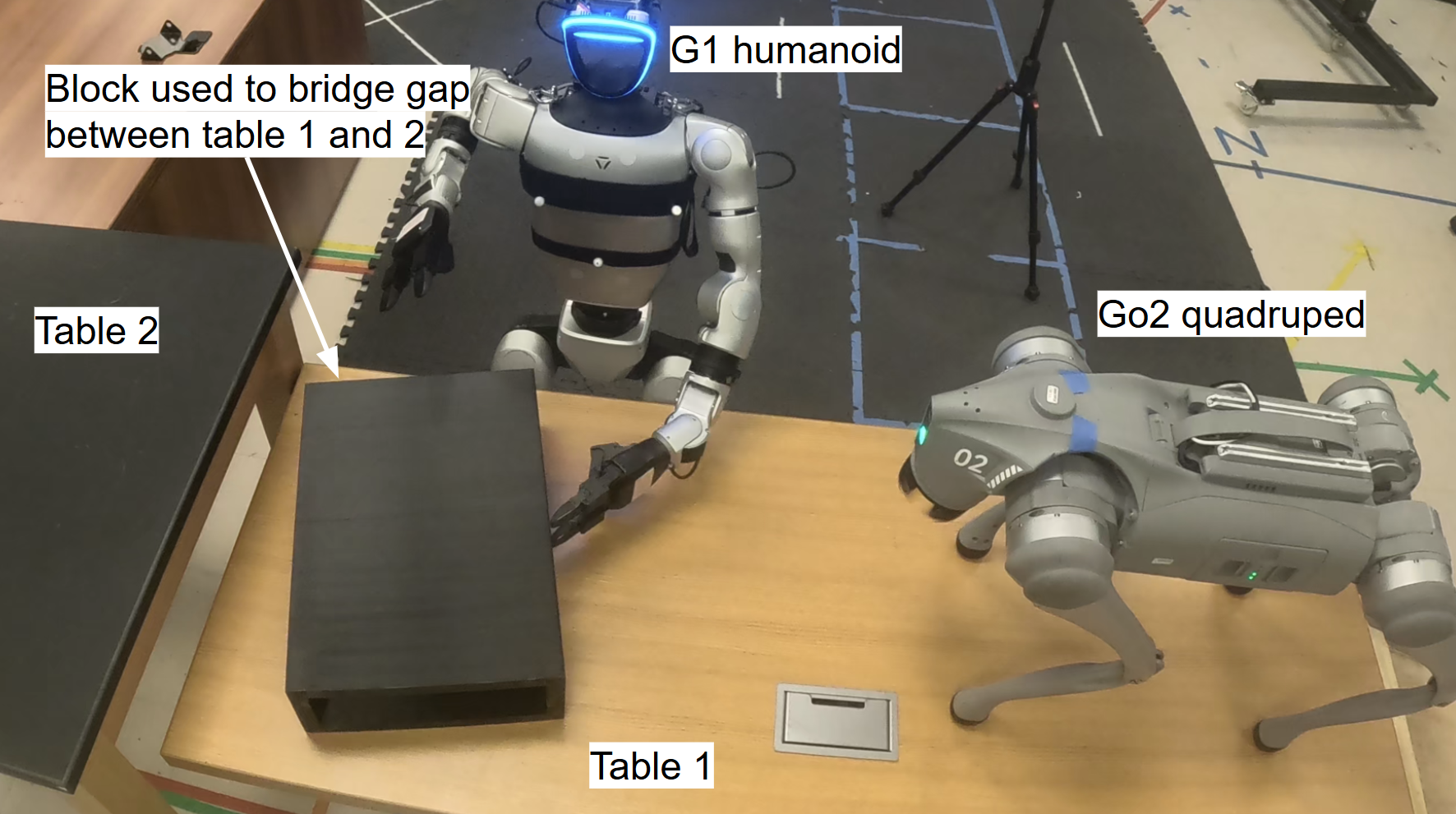}} \hfill
    \subfigure[]{\label{fig:hw3}\includegraphics[width=0.3\textwidth]{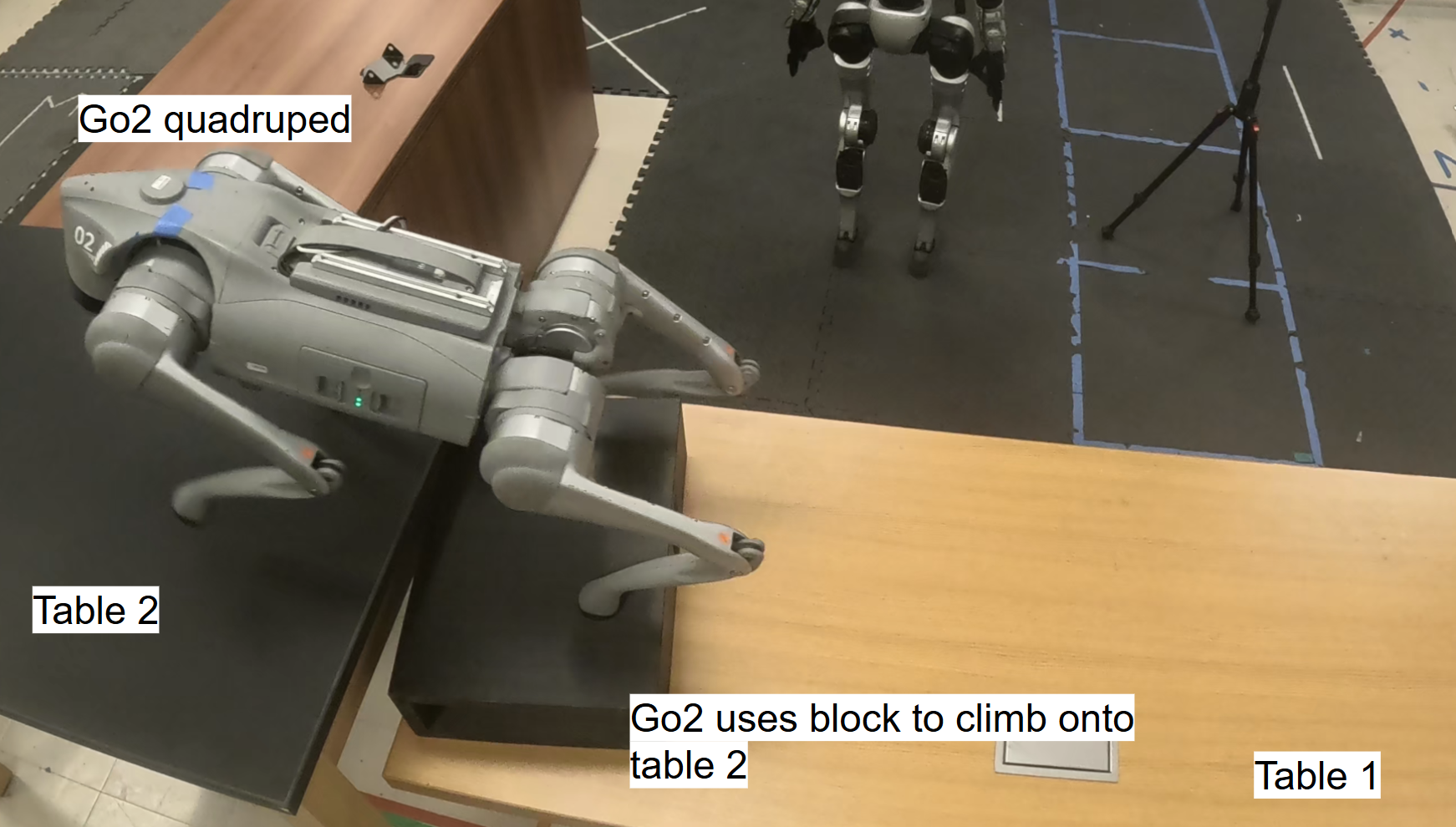}} \hfill
    \subfigure[]{\label{fig:hw4}\includegraphics[width=0.3\textwidth]{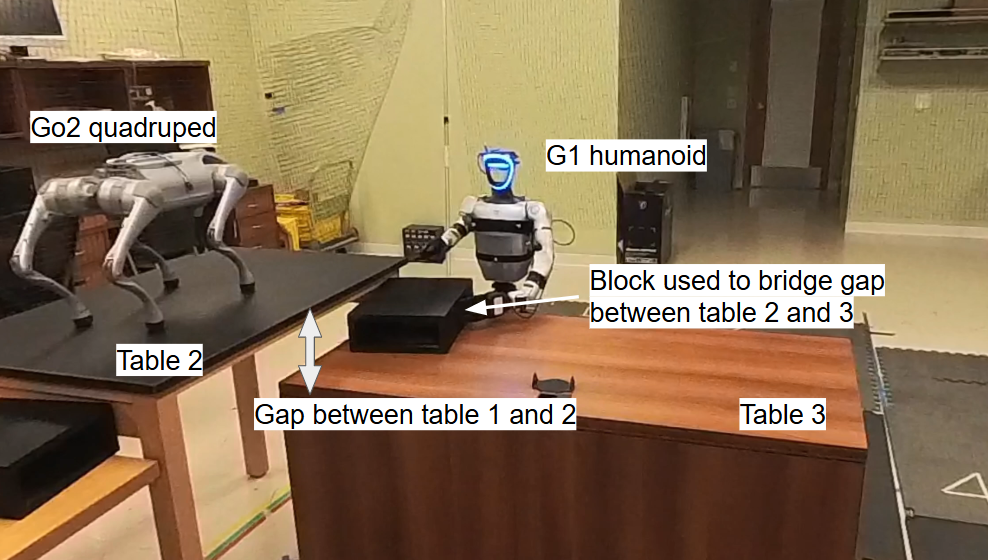}} \hfill
    \subfigure[]{\label{fig:hw5}\includegraphics[width=0.3\textwidth]{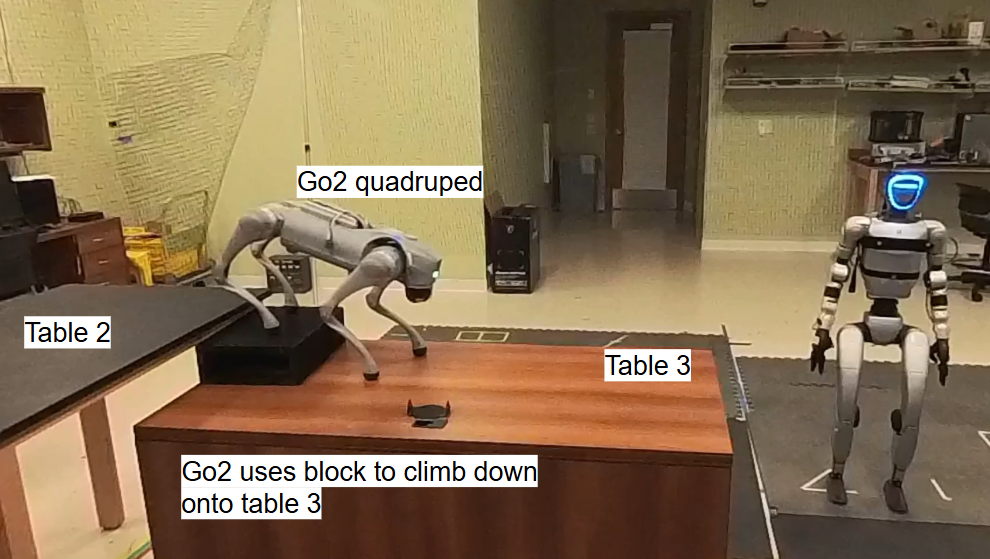}} \hfill
    \subfigure[]{\label{fig:hw6}\includegraphics[width=0.3\textwidth]{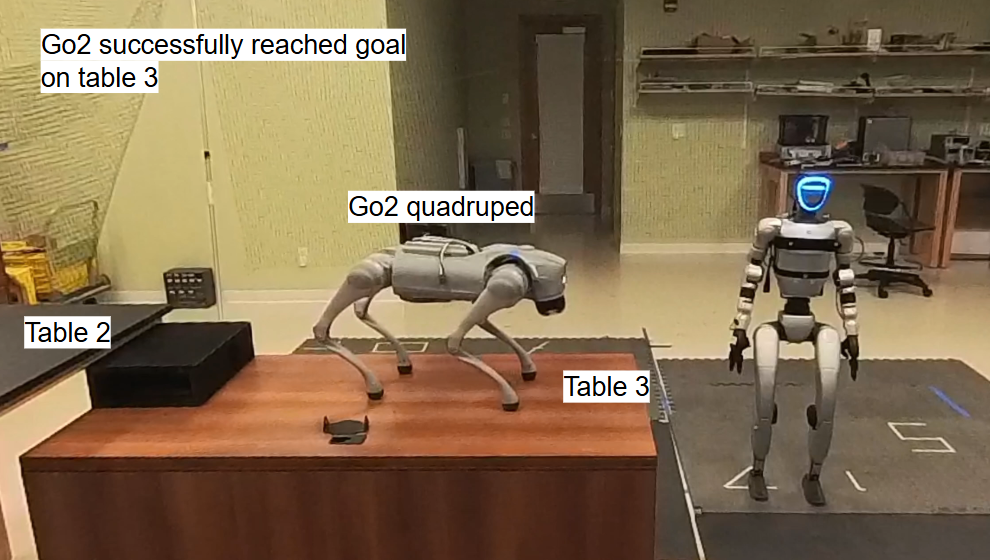}} \hfill

    \caption{Hardware demonstration.
    Fig. \ref{fig:hw1}: The Go2 reaches Table~1, but cannot climb on Table~2. Fig. \ref{fig:hw2}: The humanoid then places the first block to bridge the gap between tables 1 and 2. Fig. \ref{fig:hw3}: The quadruped uses the block to climb on to table 2. 
    \textcolor{black}{Figs. \ref{fig:hw4}-\ref{fig:hw6}: Similarly the humanoid then bridges the Table~2–Table~3 gap with a second block, enabling the Go2 to reach Table~3.}
    }

    \label{fig:hardware_demo}
\end{figure*}



\subsection{Hardware Demonstrations}\label{subsec:cs3}

\textcolor{black}{\textbf{BRiDGE}} is developed for a single robot with navigation and manipulation capabilities encoded by the \texttt{NavigReach} and \texttt{ManipReach} functions. In this section, we demonstrate, via hardware experiments, how these capabilities can be split across two robots: (i) a Unitree Go2 quadruped, which can traverse and climb surfaces satisfying $\delta_H = 15$ cm but cannot grasp objects, and (ii) a Unitree G1 humanoid, which can reliably pick and place blocks but is not equipped with a multi-plane climbing control policy. The planner is instantiated on this heterogeneous pair, with the Go2 executing all navigation actions and the G1 performing all manipulation actions. Robot localization is provided by an \emph{OptiTrack} motion-capture system for real-time 6-DoF tracking.

The workspace consists of four planes: the ground and three non-movable tables; see Fig. \ref{fig:hardware_setup}. Table~1 is connected to the ground via a ramp, and Table~2 lies between Tables 1 and 3 as the tallest one. The height differences from Table~1 to Table~2 and from Table~2 to Table~3 are both approximately $20$ cm. Two 3D-printed movable blocks $b_1$ and $b_2$, of height $12$ cm are placed on Table~4, which lies outside Go2’s workspace. The Go2 is initially located on the ground plane and is tasked with reaching a goal region located on Table~3. Although the Go2 can climb from the ground plane to Table~1 via the ramp, it cannot climb from Table~1 to Table~2 or from Table~2 to Table~3, given their height differences, as the robot can reliably climb vertical steps of up to $\delta_H = 15$ cm. Thus, reaching the goal region requires using the two movable blocks to bridge the gaps between the tables. Recall that the Go2 does not have any manipulation capabilities. However, the G1 can grab both movable blocks and place them on any non-movable plane, enabling Go2 to access initially unreachable surfaces.

To map our proposed single-robot planner to this two-robot setup, we define (i) $\texttt{ManipReach}(b_i, \Omega)$ as the full workspace, and (ii) $\texttt{NavigReach}$ over the ground and Tables 1–3, enforcing the climbing constraint $\delta_H = 15$. Conceptually, this models the Go2—the primary robot in the planner—as being able to grab any block from its current configuration and place it on any plane. In physical execution, these manipulation actions are carried out by the G1 on behalf of the Go2. This allows the heterogeneous hardware pair (Go2 for navigation, G1 for manipulation) to realize the capabilities of the single abstract agent assumed in Algorithm~1 without modifying the planner itself.

\textcolor{black}{We run Algorithm~1 without any high-level biasing: node selection is uniform and triplets are drawn uniformly from the candidate set. The only exception is block-configuration sampling, where we retain the placement bias and set $p_{\text{gap}}=0.9$.} 
The planner returns a feasible plan in 8.1~seconds, consisting of two high-level $\texttt{Move}(\bbb_i,\bbb_i')$ actions requiring first to move  $b_1$ from Table~4 to Table~1, bridging the gap between Table~1 and Table~2, and second to block $b_2$ from Table~4 to Table~3, bridging the gap between Table~2 and Table~3. Thus, during the execution of this plan, the Go2 moves to Table~2 (Fig. \ref{fig:hw1}); the G1 bridges the gap between Tables 1 and 2 using $b_1$ (Fig. \ref{fig:hw2}); the Go2 climbs to Table~2 via $b_1$ (Fig. \ref{fig:hw3}); the G1 bridges the gap between Tables 2 and 3 using $b_2$ (Fig. \ref{fig:hw4}); the Go2 moves to Table~3 via $b_2$ (Fig. \ref{fig:hw5}), reaching its goal area (Fig. \ref{fig:hw6}).

\textbf{Practical Lessons and Current Limitations:} \textcolor{black}{Our hardware demonstration also highlighted several practical limitations that are not yet addressed in the current implementation. \textit{First}, our system assumes that the poses of all blocks and planes are perfectly known. In practice, reliable manipulation would benefit from local perception feedback to refine the relative pose of each object with respect to the robot during grasping and placement, enabling closed-loop correction and more robust execution.
\textit{Second}, our proposed framework assumes that every sampled block placement is statically stable once placed which may not hold in practice especially in environments with uneven and possibly slippery terrains. Relaxing this assumption is nontrivial: stability depends on surface geometry, friction, and mass distribution, and verifying it online requires substantially richer models. Recent work integrates physics simulators into sampling-based planning to evaluate the stability of candidate placements \cite{lee2023object}, and incorporating such checks into our framework is a promising direction for real-world deployment.
\textit{Finally}, implementing our framework exclusively on a humanoid robot remains a major challenge mostly due to the requirement of climbing reliably onto blocks. Executing such behaviors requires reasoning about the precise relative configuration between the robot and the placed block, along with generating feasible footstep trajectories that achieve a stable ascent. Although our hardware experiment delegated all climbing to the Go2 quadruped, extending the system to support humanoid climbing would require integrating whole-body motion planning and contact reasoning.
These limitations reflect natural extensions of our current work, and each represents a clear avenue for future work toward closing the gap between our developed planning capabilities and full real-world autonomy.}

\section{Conclusion} \label{sec:Concl}
\textcolor{black}{In this paper, we considered planning problems for mobile robots tasked with reaching desired goal regions in environments consisting of multiple elevated navigable and static planes and movable objects. Our focus was on initially infeasible tasks, i.e., tasks where the goal region resides on an elevated plane that is unreachable. To enable the robot to reach such goal regions,} \textcolor{black}{we developed a \textcolor{black}{sampling-based} planner for reconfiguring disconnected 3D environments by strategically relocating movable objects to create new traversable connections. The method leverages a symbolic high-level planner, BFS or an LLM, to bias sampling toward promising block placements. We showed that the proposed method is probabilistically complete. Our experiments demonstrate that the approach efficiently solves long-horizon problems across a wide range of complexities and remains robust even when the high-level symbolic guidance is partially incorrect. We further validated the framework through hardware demonstrations using humanoid and quadruped robots. Future work will focus on integrating perception-driven manipulation, stability-aware block placement \textcolor{black}{in uneven terrains}, and extending the system to more general multi-robot capabilities.}



\bibliographystyle{IEEEtran}
\bibliography{SK_bib.bib}

\end{document}